\documentclass[12pt,letterpaper]{amsart}
\usepackage{hyperref}
\usepackage{graphicx, color}
\usepackage{mathtools}
\usepackage{enumerate,enumitem}
\usepackage{bbm}
\usepackage[margin=1.0in]{geometry}
\usepackage{amssymb}
\usepackage[linesnumbered,ruled,vlined]{algorithm2e}
\usepackage{subcaption}
\usepackage{float}
\usepackage{tikz}
\usepackage{yhmath}
\usepackage{multirow}
\usepackage[numbers]{natbib}

\usepackage[flushleft]{threeparttable}
\usepackage{array,booktabs,makecell}
\usepackage{microtype}

\usepackage[capitalize,noabbrev]{cleveref}

\newtheorem{theorem}{Theorem}[section]

\newtheorem{lemma}[theorem]{Lemma}

\theoremstyle{definition}

\theoremstyle{remark}
\newtheorem{remark}[theorem]{Remark}

\DeclareMathOperator{\argmin}{argmin}
\DeclareMathOperator{\sign}{sign}

\newcommand{\R}{\mathbb{R}}
\newcommand{\E}{\mathbb{E}}
\newcommand{\A}{\mathcal{A}}
\newcommand{\F}{\mathcal{F}}
\newcommand{\prob}{\mathrm{P}}

\SetKwInput{KwInput}{Input}                % Set the Input
\SetKwInput{KwOutput}{Output}              
\SetKwInput{KwIni}{Initialize Parameters}

\begin{document}
\title{Post-training Quantization for Neural Networks with Provable Guarantees}
\author{Jinjie~Zhang}
\address{Department of Mathematics, University of California San Diego}
\email{jiz003@ucsd.edu}

\author{Yixuan Zhou}
\address{Department of Mathematics, University of California San Diego}
\email{yiz044@ucsd.edu}

\author{Rayan~Saab}
\address{Department of Mathematics and Hal{\i}c{\i}o\u{g}lu Data Science Institute, University of California San Diego}
\email{rsaab@ucsd.edu}
\maketitle

\begin{abstract}
While neural networks have been remarkably successful in a wide array of applications, implementing them in resource-constrained hardware remains an area of intense research. By replacing the weights of a neural network with quantized (e.g., 4-bit, or binary) counterparts, massive savings in computation cost, memory, and power consumption are attained. To that end, we generalize a post-training neural-network quantization method, GPFQ, that is based on a greedy path-following mechanism. Among other things, we propose modifications to promote sparsity of the weights, and rigorously analyze the associated error. Additionally, our error analysis expands the results of previous work on GPFQ to handle general quantization alphabets, showing that for quantizing a single-layer network, the relative square error essentially decays linearly in the number of weights --  i.e., level of over-parametrization. Our result holds across a range of input distributions and for both fully-connected and convolutional architectures thereby also extending previous results. To empirically evaluate the method, we quantize several common architectures with few bits per weight, and test them on ImageNet, showing only minor loss of accuracy compared to unquantized models. We also demonstrate that standard modifications, such as bias correction and mixed precision quantization, further improve accuracy. 
\end{abstract}

\section{Introduction}
Over the past decade, deep neural networks (DNNs) have achieved great success in many challenging tasks, such as computer vision, natural language processing, and autonomous vehicles. Nevertheless, over-parameterized DNNs are computationally expensive to train, memory intensive to store, and energy consuming to apply. This hinders the deployment of DNNs to resource-limited applications. Therefore, model compression without significant performance degradation is an important active area of deep learning research \cite{guo2018survey, deng2020model, gholami2021survey}. One prominent approach to compression is \emph{quantization}. Here, rather than adopt a $32$-bit floating point format for the model parameters, one uses significantly fewer bits for representing weights, activations, and even gradients. Since the floating-point operations are substituted by more efficient low-bit operations, quantization can reduce inference time and power consumption. 

Following  \cite{krishnamoorthi2018quantizing}, we can classify quantization methods into two categories: \emph{quantization-aware training} and \emph{post-training quantization}. The fundamental difficulty in quantization-aware training stems from the fact that it reduces to an integer programming problem with a non-convex loss function, making it NP-hard in general. Nevertheless, many well-performing heuristic methods exist, e.g.,  \cite{courbariaux2015binaryconnect, han2015deep, zhou2017incremental, jacob2018quantization, zhang2018lq, louizos2018relaxed, wang2020sparsity}. Here one, for example, either modifies the training procedure to produce quantized weights, or successively quantizes each layer and then retrains the subsequent layers. Retraining is a powerful, albeit computationally intensive way to compensate for the accuracy loss resulting from quantization and it remains generally difficult to analyze rigorously. 

Hence, much attention has recently been dedicated to post-training quantization schemes, which directly quantize pretrained DNNs having real-valued weights, without retraining. These quantization methods either rely on a small amount of data \cite{banner2018post, choukroun2019low, zhao2019improving, nagel2020up, hubara2020improving, wang2020towards, li2021brecq, lybrand2021greedy} or can be implemented without accessing training data, i.e. data-free compression \cite{nagel2019data, cai2020zeroq, xu2020generative, liu2021zero}. 

\subsection{Related Work}\label{sec:related-work}
We now summarize some prior work on post-training quantization methods. The majority of these methods aim to reduce quantization error by minimizing a mean squared error (MSE) objective, e.g. $\min_{\alpha>0} \biggl\|W- \alpha \biggl\lfloor\frac{W}{\alpha}\biggr\rceil\biggr\|_F$,
where $W$ is a weight matrix and $\lfloor \cdot\rceil$ is a round-off operator that represents a map from the set of real numbers to the low-bit alphabet. Generally $\lfloor \cdot\rceil$ simply assigns numbers in different intervals or ``bins" to different elements of the alphabet.  Algorithms in the literature differ in their choice of $\lfloor \cdot\rceil$, as they use different strategies for determining the quantization bins. However, they share the property that once the quantization bins are selected, weights are quantized independently of each other.  For example, \citet{banner2018post} (see also \cite{zhao2019improving}) choose the thresholds to minimize a MSE metric. Their numerical results also show that for convolutional networks using different quantization thresholds ``per-channel" and bias correction can improve the accuracy of quantized models. \citet{choukroun2019low} solve a minimum mean squared error (MMSE) problem for both weights and activations quantization. Based on a small calibration data set, \citet{hubara2020improving} suggest a per-layer optimization method followed by integer programming to determine the bit-width of different layers. A bit-split and stitching technique is used by \cite{wang2020towards} that ``splits" integers into multiple bits, then optimizes each bit, and finally stitches all bits back to integers. \citet{li2021brecq} leverage the basic building blocks in DNNs and reconstructs them one-by-one. As for data-free model quantization, there are different strategies, such as weight equalization \cite{nagel2019data}, reconstructing calibration data samples according to batch normalization statistics (BNS) \cite{cai2020zeroq, xu2020generative}, and adversarial learning \cite{liu2021zero}.

\subsection{Contribution}\label{sec:contribution}
In spite of reasonable heuristic explanations and empirical results, all quantization methods mentioned in \Cref{sec:related-work} lack rigorous theoretical guarantees. Recently,  \citet{lybrand2021greedy} proposed and analyzed a method for quantizing the weights of pretrained DNNs called \emph{greedy path following quantization} (GPFQ), see \Cref{sec:GPFQ} for details. In this paper, we substantially improve GPFQ's  theoretical analysis, propose a modification to handle convolutional layers, and propose a sparsity promoting version to encourage the algorithm to set many of the weights to zero. 
We demonstrate that the performance of our quantization methods is not only good in experimental settings, but, equally importantly, has favorable and rigorous error guarantees. Specifically, the contributions of this paper are threefold: 

\textbf{1.} We generalize the results of  \cite{lybrand2021greedy} in several directions. Indeed, the results of \cite{lybrand2021greedy} apply only to alphabets, $\mathcal{A}$, of the form $\mathcal{A}=\{0,\pm 1\}$ and standard Gaussian input because the proof technique in \cite{lybrand2021greedy} relies heavily on properties of Gaussians and case-work over elements of the alphabet. It also requires the assumption that floating point weights are $\epsilon$-away from alphabet elements. In contrast, by using a different and more natural proof technique, our results avoid this assumption and extend to general alphabets like $\mathcal{A}$ in \eqref{eq-alphabet-midtread} and make the main result in \cite{lybrand2021greedy} a special case of our \cref{thm-cluster}, which in turn follows from \cref{thm-main-result-single}. Moreover, we extend the class of input vectors for which the theory applies. For example, in \Cref{sec:theory}, we show that if the input data $X\in\R^{m\times N_0}$ is either bounded or drawn from a mixture of Gaussians, then the relative square error of quantizing a neuron $w\in \R^{N_0}$ satisfies the following inequality with high probability:
\begin{equation}
    \label{eq: error-bound-mixture}
    \frac{\|Xw-Xq\|_2^2}{\|Xw\|_2^2}\lesssim \frac{m\log N_0}{N_0}
\end{equation}
where $q\in\A^{N_0}$ is the quantized weights. 
% $\delta>0$ is the step size in \eqref{eq-alphabet-midtread} and 
A mixture of Gaussians is a reasonable model for the output of some of the deeper layers in neural networks that focus on classification, thus our results are relevant in those contexts. Further, to handle convolutional neural networks (CNNs), we introduce a modification to GPFQ in \Cref{sec:exp-setup} that relies on random subsampling to make quantizing DNNs practically feasible with large batch size $m$. This also allows us to obtain quantization error bounds that resemble \eqref{eq: error-bound-mixture}, for single-layer CNNs in \Cref{sec:CNNs}.

\textbf{2.} In order to reduce the storage, computational, and power requirements of DNNs one complimentary approach  to quantization is to sparsify the weights, i.e., set many of them to zero.  In \Cref{sec:sparse-GPFQ}, we propose modifications to GPFQ that leverage soft and hard thresholding to increase sparsity of the weights of the \emph{quantized} neural networks. We present error bounds, similar to the ones in \cref{thm-main-result-single}, and provide their proofs in \Cref{sec:theory-sparse-quantization}. 

\textbf{3.} We provide extensive numerical experiments to illustrate the performance of GPFQ and its proposed modifications on common computer vision DNNs. First, we provide comparisons with other post-training quantization approaches (\Cref{sec:experiments}) and show that GPFQ achieves near-original model performance using $4$ bits and that the results for $5$ bits are competitive with state-of-the-art methods. Our experiments also demonstrate that GPFQ is compatible with various ad-hoc performance enhancing modifications such as bias correction \cite{banner2018post}, unquantizing the last layer \cite{zhou2017incremental, li2021brecq}, and mixed precision \cite{dong2019hawq, cai2020zeroq}. To illustrate the effects of sparsity, we further explore the interactions among prediction accuracy, sparsity of the weights, and regularization strength in our numerical experiments. Our results show that one can achieve near-original model performance even when half the weights (or more) are quantized to zero.

\section{Preliminaries}\label{sec:preliminaries}
In this section, we first introduce the notation that will be used throughout this paper and then recall the original GPFQ algorithm in \cite{lybrand2021greedy}.

\subsection{Notation} 
Various positive absolute constants are denoted by C, c. We use $a\lesssim b$ as shorthand for  $a\leq Cb$, and $a\gtrsim b$ for $a\geq Cb$. Let $S\subseteq\R^n$ be a Borel set. $\mathrm{Unif}(S)$ denotes the uniform distribution over $S$. An $L$-layer multi-layer perceptron, $\Phi$, acts on a vector $x\in\mathbb{R}^{N_0}$ via
\begin{equation}\label{eq-mlp}
\Phi(x):=\varphi^{(L)} \circ A^{(L)} \circ\cdots\circ \varphi^{(1)} \circ A^{(1)}(x)
\end{equation}
where $\varphi^{(i)}:\R^{N_i} \to \R^{N_i}$ is an activation function acting entrywise, and $A^{(i)}:\R^{N_{i-1}}\to \R^{N_i}$ is an affine map given by $A^{(i)}(z):= W^{(i)\top}z+b^{(i)}$. Here,   $W^{(i)}\in\R^{N_{i-1}\times N_i}$ is a weight matrix and $b^{(i)}\in\R^{N_i}$ is a bias vector. Since $w^\top x + b = \langle (w, b), (x,1)\rangle$, the bias term $b^{(i)}$ can  be treated as an extra row to the weight matrix $W^{(i)}$, so we will henceforth ignore it. We focus on \emph{midtread} alphabets
\begin{equation}\label{eq-alphabet-midtread}
\mathcal{A}= \mathcal{A}_K^\delta:= \{ \pm k\delta: 0\leq k\leq K, k\in \mathbb{Z}\}
\end{equation}
and their variants
\begin{equation}\label{eq-alphabet-midtread-variant}
\widetilde{\mathcal{A}} =\A_K^{\delta,\lambda}:=\{0\}\cup\{\pm (\lambda+k\delta): 0\leq k \leq K, k\in\mathbb{Z}\}
\end{equation}
where $\delta>0$ denotes the quantization step size and $\lambda>0$ is a threshold. For example, $\A_1^1=\{0, \pm 1\}$ is a ternary alphabet. Moreover, for alphabet $\mathcal{A}=\A_K^\delta$, we define the associated \emph{memoryless scalar quantizer} (MSQ) $\mathcal{Q}:\mathbb{R}\rightarrow \mathcal{A}$ by
\begin{equation}\label{eq:MSQ-midtread}
\mathcal{Q}(z):=\argmin_{p\in\mathcal{A}}|z-p|=\delta \sign(z)\min\biggl\{\biggl| \biggl\lfloor \frac{z}{\delta} +\frac{1}{2} \biggr \rfloor\biggr|, K \biggr\}.
\end{equation}
Further, the MSQ over $\widetilde{\A}=\A_K^{\delta,\lambda}$ is given by
\begin{align}\label{eq:MSQ-midtread-variant}
\widetilde{\mathcal{Q}}(z): &=
\begin{cases}
0 & \text{if} \ |z|\leq \lambda ,\\
\arg\min\limits_{p\in\widetilde{\mathcal{A}}}|z-p| &\text{otherwise}
\end{cases} \nonumber  \\
&= \mathbbm{1}_{\{|z|>\lambda\}} \sign(z)\biggl( \lambda+\delta\min\biggl\{\biggl|\biggl\lfloor \frac{s_\lambda(z)}{\delta} +\frac{1}{2}\biggr\rfloor \biggr|, K \biggr\} \biggr).
\end{align}
Here, 
$s_\lambda(z):=\sign(z)\max\{|z|-\lambda, 0\}$ is the \emph{soft thresholding} function and its counterpart, \emph{hard thresholding} function, is defined by
\[
h_\lambda(z):=z\mathbbm{1}_{ \{|z|>\lambda\}}= 
\begin{cases}
z &\text{if} \ |z|>\lambda, \\
0 &\text{otherwise}.
\end{cases}
\]

\begin{algorithm}[ht]
   \caption{Using GPFQ to quantize MLPs}
   \label{algorithm}
\DontPrintSemicolon
  \KwInput{A $L$-layer MLP $\Phi$ with weight matrices $W^{(i)}\in\mathbb{R}^{N_{i-1}\times N_i}$, input mini-batches $\{X_i\}_{i=1}^L\subset\mathbb{R}^{m\times N_0}$} 
  \For{$i= 1$ \KwTo $L$}
  {
  \textbf{Phase I}: Forward propagation \\
  Generate $X^{(i-1)}= \Phi^{(i-1)}(X_i)\in\mathbb{R}^{m\times N_{i-1}}$ and $\widetilde{X}^{(i-1)}= \widetilde{\Phi}^{(i-1)}(X_i)\in\mathbb{R}^{m\times N_{i-1}}$ \\
  \textbf{Phase II}: Parallel quantization for $W^{(i)}$ \\
  \textbf{repeat}\\
  Pick a column (neuron) $w\in\R^{N_{i-1}}$ of $W^{(i)}$ and set $u_0=0\in\R^m$ \\
  \For{$t= 1$ \KwTo $N_{i-1}$} 
  {
  Implement \eqref{eq-quantizer} and $u_t = u_{t-1} + w_t X^{(i-1)}_t - q_t \widetilde{X}^{(i-1)}_t$
  }
  \textbf{until} All columns of $W^{(i)}$ are quantized \\
  Obtain quantized $i$-th layer $Q^{(i)}\in\mathcal{A}^{N_{i-1}\times N_i}$ \\
  } 
\KwOutput{Quantized neural network $\widetilde{\Phi}$}
\end{algorithm}

\subsection{GPFQ}\label{sec:GPFQ}
Given a data set $X\in\mathbb{R}^{m\times N_0}$ with vectorized data stored as rows and a trained neural network $\Phi$ with weight matrices $W^{(i)}$, the GPFQ algorithm \cite{lybrand2021greedy} is a map $W^{(i)}\to Q^{(i)} \in \mathcal{A}^{N_{i-1} \times N_i}$, giving a new quantized neural network $\widetilde{\Phi}$ with  $\widetilde{\Phi}(X)\approx \Phi(X)$. The matrices $W^{(1)}, \ldots, W^{(L)}$ are quantized sequentially and in each layer every neuron (a column of $W^{(i)}$) is quantized independently of other neurons, which allows parallel quantization across neurons in a layer. 

Thus, GPFQ can be implemented recursively. Let $\Phi^{(i)}$, $\widetilde{\Phi}^{(i)}$ denote the original and quantized neural networks up to layer $i$ respectively. Assume the first $i-1$ layers have been quantized and define $X^{(i-1)}:=\Phi^{(i-1)}(X)$, $\widetilde{X}^{(i-1)}:=\widetilde{\Phi}^{(i-1)}(X)\in\mathbb{R}^{m\times N_{i-1}}$. Then  each neuron $w\in\mathbb{R}^{N_{i-1}}$ in layer $i$ is quantized by constructing $q\in\mathcal{A}^{N_{i-1}}$ such that 
\[
\widetilde{X}^{(i-1)}q=\sum_{t=1}^{N_{i-1}}q_t \widetilde{X}^{(i-1)}_t\approx \sum_{t=1}^{N_{i-1}}w_t X^{(i-1)}_t=X^{(i-1)}w
\]
where $X^{(i-1)}_t$, $\widetilde{X}^{(i-1)}_t$ are the $t$-th columns of $X^{(i-1)}$, $\widetilde{X}^{(i-1)}$. This is done by selecting $q_t$, for $t=1,2,\ldots,N_{i-1}$, so the running sum $\sum_{j=1}^{t}q_j \widetilde{X}^{(i-1)}_j$ tracks its analog $\sum_{j=1}^t w_j X^{(i-1)}_j$ as well as possible in an $\ell_2$ sense. So,
\begin{equation}\label{eq-update-rule}
q_t = \arg\min\limits_{p\in\mathcal{A}} \Bigl\|\sum_{j=1}^t w_jX_j^{(i-1)}-\sum_{j=1}^{t-1}q_j \widetilde{X}_j^{(i-1)}-p\widetilde{X}_t^{(i-1)} \Bigr\|_2^2.
\end{equation}

This is equivalent to the following iteration, which facilitates the analysis of the approximation error:
\begin{equation}\label{eq-update-rule-dynamic}
\begin{cases}
u_0 = 0 \in\mathbb{R}^m,\\
q_t = \argmin_{p\in\mathcal{A}}\bigl\| u_{t-1} +w_t X_t^{(i-1)}-p\widetilde{X}^{(i-1)}_t\bigr\|_2^2, \\
u_t = u_{t-1} + w_t X^{(i-1)}_t - q_t \widetilde{X}^{(i-1)}_t.
\end{cases}
\end{equation}
By induction, one can verify that $u_t=\sum_{j=1}^t(w_jX^{(i-1)}_j - q_j\widetilde{X}^{(i-1)}_j)$ for $t=0,1,\ldots, N_{i-1}$, and thus $\|u_{N_{i-1}}\|_2 = \|X^{(i-1)}w-\widetilde{X}^{(i-1)}q \|_2$. Moreover, one can derive a closed-form expression of $q_t$ in \eqref{eq-update-rule-dynamic} as 
\begin{equation}\label{eq-quantizer}
q_t = \mathcal{Q}\biggl(\frac{\langle \widetilde{X}_t^{(i-1)}, u_{t-1}+w_t X_t^{(i-1)} \rangle}{\|\widetilde{X}^{(i-1)}_t\|_2^2} \biggr),
\end{equation}
which is proved in \cref{lemma-update-rule-dynamic}. The whole algorithm for quantizing multilayer perceptrons (MLPs) is summarized in \cref{algorithm}. For the $i$-th layer,  this parallelizable algorithm has run time complexity $O(mN_{i-1})$ per neuron. Note that in order to quantize convolutional neural networks (CNNs), one can simply vectorize the sliding (convolutional) kernels and unfold, i.e., vectorize, the corresponding image patches. Then, taking the usual inner product on vectors, one can reduce to the case of MLPs, also see \Cref{sec:CNNs}.

\section{New Theoretical Results for GPFQ}\label{sec:theory}
\noindent In this section, we present error bounds for GPFQ with single-layer networks $\Phi$ in \eqref{eq-mlp} with $L=1$. Since the error bounds associated with the sparse GPFQ in \eqref{eq-update-rule-soft} and \eqref{eq-update-rule-hard} are very similar to the one we have for \eqref{eq-quantizer}, we focus on original GPFQ here and leave the theoretical analysis for sparse GPFQ to \Cref{sec:theory-sparse-quantization}.

In the single-layer case, we quantize the weight matrix $W:=W^{(1)}\in\mathbb{R}^{N_0\times N_1}$ and implement \eqref{eq-update-rule-dynamic} and \eqref{eq-quantizer} using $i=1$. Defining the input data $X:=X^{(0)}=\widetilde{X}^{(0)}\in\mathbb{R}^{m\times N_0}$, the iteration can be expressed as 
\begin{equation}\label{eq-update-rule-single}
\begin{cases}
u_0 = 0 \in\mathbb{R}^m,\\
q_t = \mathcal{Q}\bigl(w_t+\frac{X_t^\top u_{t-1}}{\|X_t\|_2^2} \bigr),\\
u_t = u_{t-1}+w_t X_t - q_t X_t.
\end{cases}
\end{equation}
Moreover, we have $u_t=\sum_{j=1}^t(w_j X_j-q_jX_j)$ for $t=1,2\ldots,N_0$. Clearly, our goal is to control $\|u_t\|_2$. In particular, given $t=N_0$, we recover the $\ell_2$ distance between full-precision and quantized pre-activations: $\|u_{N_0}\|_2=\|Xw-Xq\|_2$.

\subsection{Bounded Input Data}\label{sec:bounded-input-data}
We start with a quantization error bound where the feature vectors, i.e. columns, of the input data matrix $X\in\R^{m\times N_0}$ are bounded. This general result is then applied to data drawn uniformly from a Euclidean ball, and to Bernoulli random data, showing that the resulting relative square error due to quantization decays linearly with the width $N_0$ of the network.

\begin{theorem}[Bounded input data]\label{thm-main-result-single} Suppose that the  columns $X_t$ of  $X\in\mathbb{R}^{m\times N_0}$ are drawn independently from a probability distribution for which there exists $s\in(0,1)$ and $r>0$ such that $\|X_t\|_2 \leq r$ almost surely, and such that for all unit vector $u\in\mathbb{S}^{m-1}$ 
we have 
\begin{equation}\label{thm-main-result-single-assumption}
\E\frac{\langle X_t, u\rangle^2}{\|X_t\|_2^2} \geq s^2.
\end{equation}
Let $\mathcal{A}$ be the alphabet in \eqref{eq-alphabet-midtread} with step size $\delta>0$, and the largest element $q_{\max}$. Let
$w\in\R^{N_0}$ be the weights associated with a neuron with $\|w\|_\infty\leq q_{\mathrm{max}}$. Quantizing $w$ using \eqref{eq-update-rule-single}, we have
\begin{equation}\label{thm-main-result-single-conclusion1}
\prob\biggl(\|Xw-Xq\|_2^2\leq \frac{r^2\delta^2}{s^2}\log N_0\biggr) \geq 1- \frac{1}{N_0^2}\biggl(2+\frac{1}{\sqrt{1-s^2}}\biggr),
\end{equation}
and
\begin{equation}\label{thm-main-result-single-conclusion2}
\prob\biggl(\max_{1\leq t\leq N_0}\|u_t\|_2^2\leq \frac{r^2\delta^2}{s^2}\log N_0\biggr) \geq 1- \frac{1}{N_0}\biggl(2+\frac{1}{\sqrt{1-s^2}}\biggr).
\end{equation}
Furthermore, if the activation function $\varphi:\R\rightarrow\R$ is $\xi$-Lipschitz continuous, that is, $|\varphi(x)-\varphi(y)| \leq \xi|x-y|$ for all $x,y\in\R$, then we have 
\begin{equation}\label{thm-main-result-single-conclusion3}
\prob\biggl(\|\varphi(Xw)-\varphi(Xq)\|_2^2 \leq \frac{r^2\delta^2\xi^2}{s^2}\log N_0\biggr) 
\geq 1- \frac{1}{N_0^2}\biggl(2+\frac{1}{\sqrt{1-s^2}}\biggr).
\end{equation}
\end{theorem}

\begin{proof}
Let $\alpha>0$ and $\eta>0$. In the $t$-th step, by Markov's inequality, one can get
\begin{equation}\label{thm-main-result-single-eq1}
\prob(\|u_t\|_2^2\geq \alpha) = \prob(e^{\eta\|u_t\|_2^2}\geq e^{\eta\alpha}) 
\leq e^{-\eta\alpha}\E e^{\eta\|u_t\|_2^2}.
\end{equation}
According to Lemma~\ref{lemma-bound}, 
\begin{equation}\label{thm-main-result-single-eq2}
\E e^{\eta\|u_{t}\|_2^2}\leq \max\Bigl\{\E(e^{\frac{\eta\delta^2}{4}\|X_t\|_2^2}e^{\eta\|u_{t-1}\|_2^2(1-\cos^2\theta_t)}),
    \E e^{\eta\|u_{t-1}\|_2^2} \Bigr\}.
\end{equation}
Moreover, observing that $\|X_t\|_2^2\leq r^2$ a.s., then applying the law of total expectation, Lemma~\ref{lemma-bound} (2) with $\beta=1$, and assumption  \eqref{thm-main-result-single-assumption} sequentially, we obtain 
\begin{align*}
\E(e^{\frac{\eta\delta^2}{4}\|X_t\|_2^2}e^{\eta\|u_{t-1}\|_2^2(1-\cos^2\theta_t)} ) 
&\leq e^{\eta r^2\delta^2/4}\E e^{\eta\|u_{t-1}\|_2^2(1-\cos^2\theta_t)}  \\
&=  e^{\eta r^2\delta^2/4} \E(\E(e^{\eta\|u_{t-1}\|_2^2(1-\cos^2\theta_t)}\mid \mathcal{F}_{t-1}))  \\ 
&\leq e^{\eta r^2\delta^2/4}\E\Bigl(-\E(\cos^2\theta_t\mid\mathcal{F}_{t-1})(e^{\eta\|u_{t-1}\|_2^2}-1)+e^{\eta\|u_{t-1}\|_2^2}\Bigr)  \\
&\leq e^{\eta r^2\delta^2/4} \E(-s^2(e^{\eta\|u_{t-1}\|_2^2}-1)+e^{\eta\|u_{t-1}\|_2^2})\\
&= (1-s^2)e^{\eta r^2\delta^2/4}\E e^{\eta\|u_{t-1}\|_2^2}+s^2e^{\eta r^2\delta^2/4}
\end{align*} 
Hence, for each $t$, inequality \eqref{thm-main-result-single-eq2} becomes 
\begin{equation}\label{thm-main-result-single-eq3}
 \E e^{\eta\|u_{t}\|_2^2}\leq \max\Bigl\{a \E e^{\eta\|u_{t-1}\|_2^2}+b,
    \E e^{\eta\|u_{t-1}\|_2^2} \Bigr\}.   
\end{equation}
where $a:=(1-s^2)e^{\eta r^2\delta^2/4}$ and $b:=s^2e^{\eta r^2\delta^2/4}$. Let $t_0=|\{1\leq i\leq t: \E e^{\eta\|u_{i-1}\|_2^2}\leq a\E e^{\eta\|u_{i-1}\|_2^2}+b\}|$. Then, noting that $u_0 = 0$, the following inequality follows from \eqref{thm-main-result-single-eq3},  
\begin{equation}\label{thm-main-result-single-eq4}
     \E e^{\eta\|u_t\|_2^2} \leq a^{t_0}\E e^{\eta\|u_0\|_2^2}+b(1+a+\ldots+a^{t_0-1})
     =a^{t_0}+\frac{b(1-a^{t_0})}{1-a} \leq 1+\frac{b}{1-a}
\end{equation} 
where the last inequality holds provided that $a=(1-s^2)e^{\eta r^2\delta^2/4}<1$. Since the result above hold for all $\eta>0$ such that $(1-s^2)e^{\eta r^2\delta^2/4}<1$, we can choose $\eta=\frac{-2\log(1-s^2)}{r^2\delta^2}$. Then we get $a=(1-s^2)^{1/2}$ and $b=s^2(1-s^2)^{-1/2}$. It follows from \eqref{thm-main-result-single-eq1} and \eqref{thm-main-result-single-eq4} that 
\begin{align*} 
    \prob(\|u_t\|_2^2\geq \alpha)&\leq e^{-\eta\alpha}\biggl(1+\frac{b}{1-a}\biggr) 
    = \exp\biggl(\frac{2\alpha\log(1-s^2)}{r^2\delta^2}\biggr) \biggl(1 + \frac{s^2(1-s^2)^{-1/2}}{1-(1-s^2)^{1/2}}\biggr) \\
    &= \exp\biggl(\frac{2\alpha\log(1-s^2)}{r^2\delta^2}\biggr) \Bigl( 1+ (1-s^2)^{-1/2}(1+(1-s^2)^{1/2})\Bigr) \\
    &=\exp\biggl(\frac{2\alpha\log(1-s^2)}{r^2\delta^2}\biggr)\biggl(2+\frac{1}{\sqrt{1-s^2}}\biggr) \\
    &\leq \exp\biggl(\frac{-2\alpha s^2}{r^2\delta^2}\biggr)\biggl(2+\frac{1}{\sqrt{1-s^2}}\biggr).
\end{align*}
The last inequality can be obtained using the fact $\log(1+x)\leq x$ for all $x>-1$. Picking $\alpha=\frac{r^2\delta^2\log N_0}{s^2}$, we get 
\begin{align}\label{eqn:t_bound}
\prob\biggl(\|u_t\|_2^2\geq \frac{r^2\delta^2}{s^2}\log N_0\biggr) \leq \frac{1}{N_0^2}\biggl(2+\frac{1}{\sqrt{1-s^2}}\biggr).
\end{align}
From \eqref{eqn:t_bound} we can first deduce \eqref{thm-main-result-single-conclusion1}, by setting $t=N_0$ and using the fact $u_{N_0}=Xw-Xq$.
If the activation function $\varphi$ is $\xi$-Lipschitz, then $\|\varphi(Xw)-\varphi(Xq)\|_2\leq \xi\|Xw-Xq\|_2$ and \eqref{thm-main-result-single-conclusion1} implies \eqref{thm-main-result-single-conclusion3}. Moreover, applying a union bound over $t$ to \eqref{eqn:t_bound}, one can get \eqref{thm-main-result-single-conclusion2}.
\end{proof} 

Next, we illustrate how \cref{thm-main-result-single} can be applied to obtain error bounds associated with uniformly distributed and Bernoulli distributed input data.

\subsubsection{Uniformly Distributed Data} Let $B_r\subset\R^m$ be the closed ball with center $0$ and radius $r>0$. Suppose that columns $X_t$ of  $X\in\R^{m\times N_0}$ are drawn i.i.d. from $\mathrm{Unif}(B_r)$. Then $\|X_t\|_2\leq r$ and $Z:=X_t/\|X_t\|_2\sim \mathrm{Unif}(\mathbb{S}^{m-1})$. Since $Z$ is rotation invariant, for any unit vector $u\in\mathbb{S}^{m-1}$, we have 
$\E\Bigl\langle \frac{X_t}{\|X_t\|_2}, u\Bigr\rangle^2 =\E \langle Z,u\rangle^2=\E\langle Z, e_1\rangle^2=\E Z_1^2=\frac{1}{m}.$
The last equality holds because $\|Z\|_2=1$ and $\E Z_1^2=\E Z_2^2=\ldots=\E Z_m^2=\frac{1}{m}\E\Bigl(\sum_{i=1}^m  Z_i^2\Bigr)=\frac{1}{m}$. So Theorem~\ref{thm-main-result-single} implies that, with high probability
\begin{equation}\label{example-uniform-eq1}
\|Xw-Xq\|_2^2\leq mr^2\delta^2\log N_0.
\end{equation}
Moreover, by \cref{lemma-uniform-square}, 
$\E \|X_t\|_2^2 =\frac{mr^2}{m+2}$. It follows that $\E(X^\top X)=\E\|X_1\|_2^2  I_{N_0}=\frac{mr^2}{m+2}I_{N_0}$ and thus $\E \|Xw\|_2^2= w^\top\E(X^\top X) w = \frac{mr^2}{m+2}\|w\|_2^2.$
If the weight vector $w\in\R^{N_0}$ is \emph{generic} in the sense that $\|w\|_2^2\gtrsim N_0$, then
\begin{equation}\label{example-uniform-eq2}
    \E \|Xw\|_2^2\gtrsim \frac{mN_0 r^2}{m+2}. 
\end{equation}
Combining \eqref{example-uniform-eq1} with \eqref{example-uniform-eq2}, the relative error satisfies $\frac{\|Xw-Xq\|_2^2}{\|Xw\|_2^2} \lesssim {\frac{m\delta^2\log N_0}{N_0}}$.

\subsubsection{Data from a Symmetric Bernoulli Distribution} We say that a random vector $Z=(Z_1,Z_2,\ldots,Z_m)$ is \emph{symmetric Bernoulli} if the coordinates $Z_i$ are independent and $\prob(Z_i=1)=\prob(Z_i=-1)=\frac{1}{2}$. Now assume that columns $X_t$ of $X\in\R^{m\times N_0}$ are independent and subject to symmetric Bernoulli distribution. Clearly, $\|X_t\|_2=\sqrt{m}$. If $u\in\R^m$ is a unit vector, then
$\E\frac{\langle X_t, u\rangle^2}{\|X_t\|_2^2} =\frac{u^\top \E(X_t X_t^\top)u}{m} =\frac{\|u\|_2^2}{m}=\frac{1}{m}.$ Hence, by Theorem~\ref{thm-main-result-single}, 
\begin{equation}\label{example-bernoulli-eq1}
\|Xw-Xq\|_2^2\leq m^2\delta^2\log N_0
\end{equation}
holds with high probability. Again,  a generic $w\in\R^{N_0}$ with $\|w\|_2^2\gtrsim N_0$ satisfies $\E\|Xw\|_2^2 = w^\top \E(X^\top X)w= m\|w\|_2^2\gtrsim m N_0$
and therefore $\frac{\|Xw-Xq\|_2^2}{\|Xw\|_2^2} \lesssim {\frac{m\delta^2\log N_0}{N_0}}$.

\subsection{Gaussian Clusters}\label{sec:gaussian-clusters}
Here, we consider data drawn from Gaussian clusters, which unlike the previously considered models, are unbounded. One reason for considering Gaussian clusters is that they are a reasonable model for the activations in deeper layers of networks designed for classification. Specifically, suppose our samples are drawn from $d$ normally distributed clusters $\mathcal{K}_i:=\mathcal{N}(z^{(i)}, \sigma^2I_{N_0})$ with fixed centers $z^{(i)}\in\R^{N_0}$ and $\sigma>0$. Suppose, for simplicity, that we independently draw $n$ samples from each cluster and vertically stack them in order as rows of $X$ (this ordering does not affect our results in \cref{thm-cluster}). Let $m:=nd$. So, for $1\leq i\leq d$, the row indices of $X$ ranging from $(i-1)n+1$ to $in$ come from cluster $\mathcal{K}_i$. Then the $t$-th column of $X$ is of the form 
\begin{equation}\label{cluster-eq1}
X_t=[Y_t^{(1)}, Y_t^{(2)}, \ldots, Y_t^{(d)}]^\top \in\R^{m}
\end{equation}
where $Y_t^{(i)}\sim\mathcal{N}(z_t^{(i)}\mathbbm{1}_n,\sigma^2I_n)$.

\begin{theorem}[Gaussian clusters]\label{thm-cluster}
Let $X\in\mathbb{R}^{m\times N_0}$ be as in \eqref{cluster-eq1} and let $\mathcal{A}$ be as in \eqref{eq-alphabet-midtread}, with step size $\delta>0$ and the largest element $q_{\max}$. Let $p\in\mathbb{N}$, $K:=1+\sigma^{-2}\max_{1\leq i\leq d}\|z^{(i)}\|_\infty^2$, and
$w\in\R^{N_0}$ be the weights associated with a neuron, with $\|w\|_\infty\leq q_{\mathrm{max}}$. Quantizing $w$ using \eqref{eq-update-rule-single}, we have
%\begin{enumerate}
 %   \item
 \[\prob\Bigl(\|Xw-Xq\|_2^2\geq 4pm^2K^2\delta^2\sigma^2 \log N_0\Bigr) \lesssim \frac{\sqrt{mK}}{N_0^p},\text{\quad and}\]
 
 %   \item
 \[\prob\Bigl(\max_{1\leq t\leq N_0}\|u_t\|_2^2\geq 4pm^2K^2\delta^2\sigma^2 \log N_0\Bigr) \lesssim \frac{\sqrt{mK}}{N_0^{p-1}}.\]
%\end{enumerate}
If the activation function $\varphi$ is $\xi$-Lipschitz continuous, then 
\[
\prob\biggl(\|\varphi(Xw)-\varphi(Xq)\|_2^2\geq 4pm^2K^2\xi^2\delta^2\sigma^2 \log N_0\biggr) 
\lesssim \frac{\sqrt{mK}}{N_0^p}.  
\]
\end{theorem}
The proof of \cref{thm-cluster} can be found in \Cref{sec:thm-cluster-proof}.

\subsubsection{Normally Distributed Data} As a special case of \eqref{cluster-eq1}, let $X\in\R^{m\times N_0}$ be a Gaussian matrix with $X_{ij}\overset{i.i.d.}{\sim}\mathcal{N}(0,\sigma^2)$ corresponding to $d=1$, $n=m$, and $z^{(1)}=0$. \cref{thm-cluster} implies that $K=1$ and
\begin{equation}\label{example-normal-eq1}
\prob\biggl(\|Xw-Xq\|_2^2\geq 4pm^2\delta^2\sigma^2 \log N_0\biggr) \lesssim \frac{\sqrt{m}}{N_0^p}.
\end{equation}
Further, suppose that $w\in\R^{N_0}$ is generic, i.e. $\|w\|_2^2\gtrsim N_0$. In this case, $\E \|Xw\|_2^2=m\sigma^2\|w\|_2^2\gtrsim m\sigma^2 N_0$. So, with high probability, the relative error in our quantization satisfies 
\begin{equation}\label{example-normal-eq2}
   \frac{\|Xw-Xq\|_2^2}{\|Xw\|_2^2} \lesssim {\frac{m\delta^2\log N_0}{N_0}}. 
\end{equation}
Thus, here again, the relative square error for quantizing a single-layer MLP decays linearly (up to a log factor) in the number of neurons $N_0$.
Note that \eqref{example-normal-eq2}, for ternary alphabets, is the main result given by \cite{lybrand2021greedy}, which we now obtain as a special case of \cref{thm-cluster}.

\begin{remark}
In \Cref{sec:bounded-input-data} and \Cref{sec:gaussian-clusters}, we have shown that if the columns of $X\in\R^{m\times N_0}$ are drawn from proper distributions, then the relative error for quantization is small when $m\ll N_0$. Now consider the case where the feature vectors $\{X_t\}_{t=1}^{N_0}$ live in a $l$-dimensional subspace with $l<m$. In this case, $X=VF$ where $V\in\R^{m\times l}$  satisfies $V^\top V=I$, and the columns $F_t$ of $F\in\R^{l\times N_0}$ are drawn i.i.d. from a distribution $\mathcal{P}$. Suppose, for example, that $\mathcal{P}=\mathrm{Unif}(B_r)$. Due to $X=VF$, one can express any unit vector in the range of $X$ as $u=Vv$ with $v\in\R^l$. Then we have $1=\|u\|_2=\|Vv\|_2=\|v\|_2$, $\|X_t\|_2=\|VF_t\|_2=\|F_t\|_2\leq r$, and $\E\frac{\langle X_t, u\rangle^2}{\|X_t\|_2^2}=\E\frac{\langle VF_t, Vv\rangle^2}{\|VF_t\|_2^2}=\E\frac{\langle F_t, v\rangle^2}{\|F_t\|_2^2}= l^{-1}$ by our assumption for $\mathcal{P}$. Because $u_t$ in \cref{thm-main-result-single} is a linear combination of $X_j$, the proof of \cref{thm-main-result-single} remains unchanged if \eqref{thm-main-result-single-assumption} holds for all unit vectors $u$ in the range of $X$. It follows that \cref{thm-main-result-single} holds for $X$ with $s^2=l^{-1}$ and thus the relative error for quantizing the data in a $l$-dimensional subspace is improved to $\frac{\|Xw-Xq\|_2^2}{\|Xw\|_2^2}\leq \frac{l\delta^2\log N_0}{N_0}$. Applying a similar argument to $\mathcal{P}$ representing either a symmetric Bernoulli distribution or Gaussian distribution, one can replace $m$ in their corresponding relative errors by $l$. In short, the relative error depends not on the number of training samples $m$ but on the intrinsic dimension of the features $l$.
\end{remark}

\subsection{Convolutional Neural Networks}\label{sec:CNNs}
% \begin{figure}[ht]
% \centering
% \includegraphics[width=0.8\columnwidth]{figures/vectorization.png}
% \caption{Illustration of vectorizing local blocks and kernels when $C_\mathrm{in}=2$ and $k_1=k_2=4$.}
% \label{fig:vectorization}
% \end{figure}

In this section, we  derive error bounds for single-layer CNNs.
Let $Z\in\mathbb{R}^{B\times C_\mathrm{in}\times S_1\times S_2}$ be a mini-batch of images with batch size $B$, input channels $C_\mathrm{in}$, height $S_1$, and width $S_2$. Suppose that all entries of $Z$ are i.i.d. drawn from $\mathcal{N}(0,1)$ and suppose we have $C_\mathrm{out}$ convolutional kernels $\{w_i\}_{i=1}^{C_\mathrm{out}}\subseteq\mathbb{R}^{C_\mathrm{in}\times k_1\times k_2}$. Let these kernels ``slide" over $Z$ with fixed stride $(k_1, k_2)$ such that sliding local blocks generated by moving $w_i$ on $Z$ are disjoint. Additionally, if $T$ is the number of randomly selected sliding local blocks (in $\R^{C_\mathrm{in}\times k_1\times k_2}$) from each image, then one can vectorize all $BT$ local blocks and stack them together to obtain a single data matrix $X\in\R^{BT\times C_\mathrm{in} k_1k_2}$. Moreover, each kernel $w_i$ can be viewed as a column vector in $\R^{C_\mathrm{in}k_1k_2}$ and thus $W=[w_1, w_2,\ldots, w_{C_\mathrm{out}}]\in\R^{C_\mathrm{in}k_1k_2\times C_\mathrm{out}}$ is the weight matrix to be quantized.  Thus, we need to convert $W$ to $Q=[q_1, q_2,\ldots, q_{C_\mathrm{out}}]\in\mathcal{A}^{C_\mathrm{in}k_1k_2\times C_\mathrm{out}}$ with $XQ\approx XW$, as before. Since extracted local blocks from $Z$ are disjoint, columns of $X$ are independent and subject to $\mathcal{N}(0, I_{BT})$. Hence, one can apply \eqref{example-normal-eq1} with $m=BT$, $N_0=C_\mathrm{in}k_1k_2$, $\sigma=1$, and any $p\in\mathbb{N}$. Specifically, for $1\leq i\leq C_\mathrm{out}$, we get $\prob\Bigl(\|Xw_i-Xq_i\|^2_2\geq 4pB^2T^2\delta^2\log(C_\mathrm{in}k_1k_2)\Bigr) 
\lesssim  \frac{\sqrt{BT}}{(C_\mathrm{in}k_1k_2)^p}$.
By a union bound, $\prob \Bigl(\max_{1\leq i\leq C_\mathrm{out}}\|Xw_i-Xq_i\|^2_2\geq  4pB^2 T^2\delta^2\log(C_\mathrm{in}k_1k_2)\Bigr) \lesssim \frac{C_\mathrm{out}\sqrt{BT}}{(C_\mathrm{in}k_1k_2)^p}$.

\section{Sparse GPFQ and Error Analysis}\label{sec:sparse-GPFQ}
Having extended the results pertaining to GPFQ to cover multiple distributions of the input data, as well as general alphabets, we now propose modifications to produce quantized weights that are also sparse, i.e., that have a large fraction of coefficients being 0. Our sparse quantization schemes result from  adding a regularization term to  \eqref{eq-update-rule-dynamic}. Specifically, in order to generate sparse $q\in\A^{N_{i-1}}$, we compute $q_t$ via
\begin{equation}\label{eq:update-rule-optimization}
q_t = \arg\min_{p\in\mathcal{A}} \biggl(\frac{1}{2}\Bigl\| u_{t-1} + w_tX_t^{(i-1)}-p\widetilde{X}_t^{(i-1)} \Bigr\|_2^2+\lambda |p|\|\widetilde{X}_t^{(i-1)}\|_2^2 \biggr)
\end{equation}
where $\lambda>0$ is a regularization parameter.
Conveniently, \cref{lemma:update-rule-optimization} shows that the solution of \eqref{eq:update-rule-optimization} is given by
\begin{equation}\label{eq-update-rule-soft}
q_t = \mathcal{Q}\circ s_\lambda\biggl( \frac{\langle \widetilde{X}_t^{(i-1)}, u_{t-1}+w_t X_t^{(i-1)} \rangle}{\|\widetilde{X}^{(i-1)}_t\|_2^2}\biggr)
\end{equation}
where $s_\lambda$ denotes soft thresholding. It is then natural to consider a variant of \eqref{eq-update-rule-soft} replacing $s_\lambda$ with hard thresholding, $h_\lambda$. Since $h_\lambda(z)$ has jump discontinuities at $z=\pm\lambda$, the corresponding alphabet and quantizer should be adapted to this change. Thus, we use $\widetilde{\mathcal{Q}}(z)$ over $\widetilde{\A}=\A_K^{\delta,\lambda}$ as in \eqref{eq:MSQ-midtread-variant} and $q_t\in\widetilde{\A}$ is obtained via
\begin{equation}\label{eq-update-rule-hard}
q_t = \widetilde{\mathcal{Q}} \circ h_\lambda\biggl( \frac{\langle \widetilde{X}_t^{(i-1)}, u_{t-1}+w_t X_t^{(i-1)} \rangle}{\|\widetilde{X}^{(i-1)}_t\|_2^2}\biggr).
\end{equation}
In both cases, we update the error vector via $u_t=u_{t-1}+w_t X_t^{(i-1)} - q_t \widetilde{X}_t^{(i-1)}$, as before. In summary, for quantizing a single-layer network, similar to \eqref{eq-update-rule-single} the two sparse GPFQ schemes related to soft and hard thresholding are given by

\noindent\begin{minipage}{.5\linewidth}
\begin{equation}\label{eq-update-rule-single-soft}
\begin{cases}
u_0 = 0 \in\mathbb{R}^m,\\
q_t = \mathcal{Q}\circ s_\lambda \bigl(w_t+\frac{X_t^\top u_{t-1}}{\|X_t\|_2^2} \bigr),\\
u_t = u_{t-1}+w_t X_t - q_t X_t.
\end{cases}
\end{equation}
\end{minipage}%
\begin{minipage}{.5\linewidth}
\begin{equation}\label{eq-update-rule-single-hard}
\begin{cases}
u_0 = 0 \in\mathbb{R}^m,\\
q_t = \widetilde{\mathcal{Q}}\circ h_\lambda\bigl(w_t+\frac{X_t^\top u_{t-1}}{\|X_t\|_2^2} \bigr),\\
u_t = u_{t-1}+w_t X_t - q_t X_t.
\end{cases}
\end{equation}
\end{minipage}
Interesting, with these sparsity promoting modifications, one can prove similar error bounds to GPFQ. To illustrate with bounded or Gaussian clustered data, we show that sparse GPFQ admits similar error bounds as in \cref{thm-main-result-single} and \cref{thm-cluster}. The following results are proved in \Cref{sec:theory-sparse-quantization}.
\begin{theorem}[Sparse GPFQ with bounded input data]\label{thm-main-result-single-sparse} 
Under the conditions of \cref{thm-main-result-single}, we have the following. \\
$(a)$ Quantizing $w$ using \eqref{eq-update-rule-single-soft} with the alphabet $\A$ in \eqref{eq-alphabet-midtread}, we have 
\[
\prob\Bigl(\|Xw-Xq\|_2^2\leq \frac{r^2(2\lambda+\delta)^2}{s^2}\log N_0\Bigr) \geq 1- \frac{1}{N_0^2}\Bigl(2+\frac{1}{\sqrt{1-s^2}}\Bigr).
\]
% \item $\prob\Bigl(\max_{1\leq t\leq N_0}\|u_t\|_2^2\leq \frac{r^2(2\lambda+\delta)^2}{s^2}\log N_0\Bigr) \geq 1- \frac{1}{N_0}\Bigl(2+\frac{1}{\sqrt{1-s^2}}\Bigr)$.
$(b)$ Quantizing $w$ using \eqref{eq-update-rule-single-hard} with the alphabet $\widetilde{\A}$ in \eqref{eq-alphabet-midtread-variant}, we have
\[ 
\prob\Bigl(\|Xw-Xq\|_2^2\leq \frac{r^2\max\{2\lambda,\delta\}^2}{s^2}\log N_0\Bigr) \geq 1- \frac{1}{N_0^2}\Bigl(2+\frac{1}{\sqrt{1-s^2}}\Bigr).
\]
% \item $\prob\Bigl(\max_{1\leq t\leq N_0}\|u_t\|_2^2\leq \frac{r^2\max\{2\lambda,\delta\}^2}{s^2}\log N_0\Bigr) \geq 1- \frac{1}{N_0}\Bigl(2+\frac{1}{\sqrt{1-s^2}}\Bigr)$.
\end{theorem}

\begin{theorem}[Sparse GPFQ for Gaussian clusters]\label{thm-cluster-sparse}
Under the assumptions of \cref{thm-cluster}, the followings inequalities hold.\\
$(a)$ Quantizing $w$ using \eqref{eq-update-rule-single-soft} with the alphabet $\A$ in \eqref{eq-alphabet-midtread}, we have
\[
\prob\Bigl(\|Xw-Xq\|_2^2\geq 4pm^2K^2(2\lambda+\delta)^2\sigma^2 \log N_0\Bigr) \lesssim \frac{\sqrt{mK}}{N_0^p}.
\]
% \item $\prob\Bigl(\max_{1\leq t\leq N_0}\|u_t\|_2^2\geq 4pm^2K^2(2\lambda+\delta)^2\sigma^2 \log N_0\Bigr) \lesssim \frac{\sqrt{mK}}{N_0^{p-1}}$.
$(b)$ Quantizing $w$ using \eqref{eq-update-rule-single-hard} with the alphabet $\widetilde{\A}$ in \eqref{eq-alphabet-midtread-variant}, we have
\[
\prob\Bigl(\|Xw-Xq\|_2^2\geq 4pm^2K^2\max\{2\lambda,\delta\}^2\sigma^2 \log N_0\Bigr) \lesssim \frac{\sqrt{mK}}{N_0^p}.
\]
% \item $\prob\Bigl(\max_{1\leq t\leq N_0}\|u_t\|_2^2\geq 4pm^2K^2\max\{2\lambda, \delta\}^2\sigma^2 \log N_0\Bigr) \lesssim \frac{\sqrt{mK}}{N_0^{p-1}}$.
\end{theorem}
Note that the sparsity regularization term $\lambda$ only appears in the error bounds, making them slightly worse than those where no sparsity is enforced. 
In \Cref{sec:exp-sparse}, we will numerically explore the impact of $\lambda$ on the sparsity and accuracy of quantized neural networks. 

\section{Experiments}\label{sec:experiments}
To evaluate the performance of our method and compare it with the approaches reviewed in \Cref{sec:related-work}, we test our modified GPFQ on the ImageNet classification task \footnote{Our code for experiments is available: \href{https://github.com/YixuanSeanZhou/Quantized_Neural_Nets.git}{$\mathrm{https://github.com/YixuanSeanZhou/Quantized\_Neural\_Nets.git}$}}. In particular, we focus on ILSVRC-2012 \cite{deng2009imagenet}, a $1000$-category dataset with over $1.2$ million training images and $50$ thousand validation images. All images in ILSVRC-2012 are preprocessed in a standard manner before they are fed into neural networks:  we resize each image to $256 \times 256$ and use the normalized $224 \times 224$ center crop. The evaluation metrics we use are top-$1$ and top-$5$ accuracy of the quantized models on the validation dataset.

\subsection{Experimental Setup}\label{sec:exp-setup}
For reproducibility and fairness of comparison, we use the pretrained $32$-bit floating point neural networks provided by torchvision\footnote{\href{https://pytorch.org/vision/stable/models.html}{https://pytorch.org/vision/stable/models.html}} in PyTorch \cite{paszke2019pytorch}. We test several well-known neural network architectures including: AlexNet \cite{krizhevsky2012imagenet}, VGG-16 \cite{simonyan2014very}, GoogLeNet \cite{szegedy2015going}, ResNet-18, ResNet-50 \cite{he2016deep}, and EfficeintNet-B1 \cite{tan2019efficientnet}. 
In the following experiments, we will focus on quantizing the weights of fully-connected and convolutional layers of the above architectures, as our theory applies specifically to these types of layers\footnote{Batch normalization layers, while not explicitly covered by our methods in the preceeding sections, are  easy to handle. Indeed,
in \Cref{sec:fusion-bn}, we show that our approach can effectively quantize batch normalization layers by merging them with their preceding convolutional layers before quantization, and we demonstrate experimentally that this does not negatively impact performance.}.

Let $b\in\mathbb{N}$ denote the number of bits used for quantization. Here, we fix $b$ for all the layers. In our experiments with GPFQ, we adopt the midtread alphabets $\mathcal{A}_K^\delta$ in \eqref{eq-alphabet-midtread} with
\begin{equation}\label{eq-alphabet-exp}
    K= 2^{b-1}, \quad \delta = \frac{R}{2^{b-1}},
\end{equation}
where $R>0$ is a hyper-paramter. Indeed, according to  \eqref{eq-alphabet-midtread}, $\mathcal{A}_K^\delta$ is symmetric with maximal element  $q_\mathrm{max}=K\delta=R$. Since $b$ is fixed, all that remains is to select $R$ in \eqref{eq-alphabet-exp} based on the distribution of weights. To that end, suppose we are quantizing the $i$-th layer of a neural network with weight matrix $W^{(i)}\in\R^{N_{i-1}\times N_{i}}$. Then,  \cref{thm-main-result-single} and \cref{thm-cluster} require that  $R=q_\mathrm{max}\geq \max_{k,j} |W^{(i)}_{k, j}|$, and yield error bounds that favor a smaller step size $\delta \propto R$. In practice, however, the weights may have outliers with large magnitudes, which would entail unnecessarily using a large $R$. Thus, rather than choosing $R=\max_{k,j}|W_{k,j}^{(i)}| $, we will consider the average infinity norm of weights across all neurons $w$, i.e. columns of $W^{(i)}$. That is
$R \propto \frac{1}{N_i} \sum_{1\leq j \leq N_{i}} \|W^{(i)}_j\|_\infty$. 
Then, by \eqref{eq-alphabet-exp}, the step size used for quantizing the $i$-th layer is given by
\begin{equation}\label{eq:exp-step-size}
\delta^{(i)} := \frac{C}{2^{b-1}N_i}\sum_{1\leq j \leq N_i}\|W_j^{(i)}\|_\infty. 
\end{equation}
Here,  $C\geq 1$ is independent of $i$ and fixed across layers, batch-sizes, and bit widths. To obtain a good choice of $C$, we perform a grid search with cross-validation over the interval $[1, 2]$, albeit on a small batch size $m\leq 128$. So the tuning of $C$ takes very little time compared to the quantization with the full training data. Note that the tuning and quantization scale linearly in the size of the data set and the number of parameters of the network. This means that this entire process's computational complexity is dominated by the original training of the network and there is no problem with its scaling to large networks. Moreover, by choosing the maximal element in our alphabet, i.e. $q_{\max}=2^{b-1}\delta^{(i)}$, to be a constant $C\in [1,2]$ times the average $\ell_\infty$ norm of all the neurons, we are selecting a number that is effectively larger than most of the weights and thereby corresponding perfectly with the theory for most of the neurons. For the remaining neurons, the vast majority of the weights will be below this threshold, and only the outlier weights, in general, will exceed it. In \Cref{sec:results-large-weights}, we present a theoretical analysis of the expected error when a few weights exceed $q_{\max}$. We not only show that the proposed algorithm is still effective in this scenario, but also that in some cases, it may be beneficial to choose $\delta$ small enough such that some weights exceed $q_{\max}$. The analysis in \Cref{sec:results-large-weights} is consistent with, and helps explain the experimental results in this section. Further, we comment that a more thorough search for an optimal $C$ depending on these individual parameters, e.g. $b$, may improve performance. 

\begin{table}[ht]
\caption{Top-1/Top-5 accuracy drop using $b=5$ bits.}
\label{table:model-results}
\centering 
\begin{tabular}{lccc|lccc}
\toprule
Model & $C$ & $m$ & Acc Drop (\%) & Model & $C$ & $m$ & Acc Drop (\%) \\
\midrule 
AlexNet    & 1.1 & 2048 & 0.85/0.33 & GoogLeNet  & 1.41 &  2048 & 0.60/0.46 \\
VGG-16 &  1.0 & 512 & 0.63/0.32 & EfficientNet-B1 & 1.6 &  2048 & 0.45/0.18\\
ResNet-18 & 1.16 & 4096 & 0.49/0.23 & ResNet-50    & 1.81 & 2048 & 0.62/0.11 \\
\bottomrule
\end{tabular}
\vskip -0.1in
\end{table}

\begin{figure*}[ht]
\begin{subfigure}{.33\textwidth}
  \centering
  \includegraphics[width=\linewidth]{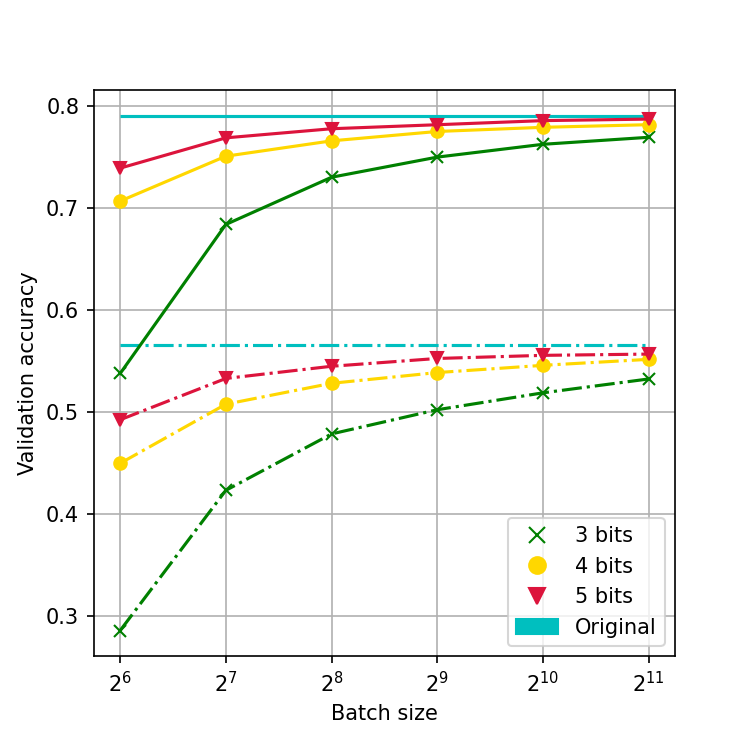}  
\caption{AlexNet}
\label{fig:Quant-AlexNet}
\end{subfigure}%
\begin{subfigure}{.33\textwidth}
  \centering
  \includegraphics[width=\linewidth]{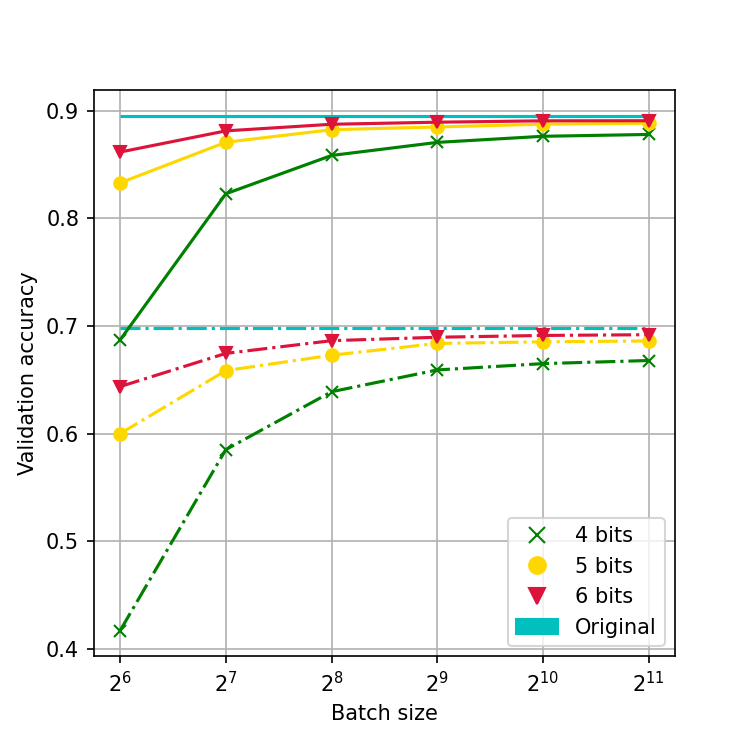}  
\caption{GoogLeNet} 
\label{fig:Quant-GoogLeNet}
\end{subfigure}%
\begin{subfigure}{.33\textwidth}
  \centering
  \includegraphics[width=\linewidth]{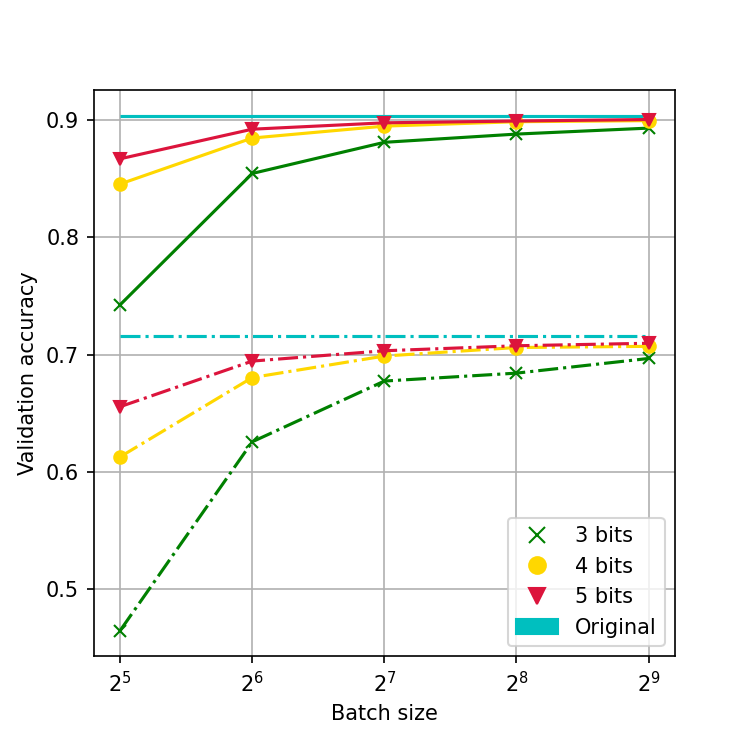}  
\caption{VGG-16}
\label{fig:Quant-VGG16}
\end{subfigure}\\
\begin{subfigure}{.33\textwidth}
  \centering
  \includegraphics[width=\linewidth]{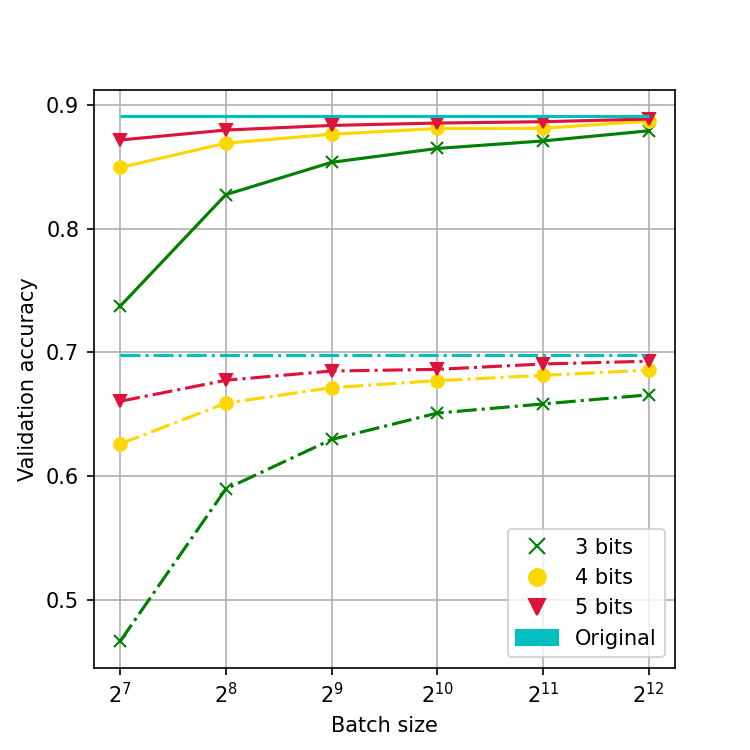}  
\caption{ResNet-18} 
\label{fig:Quant-ResNet18}
\end{subfigure}%
\begin{subfigure}{.33\textwidth}
  \centering
  \includegraphics[width=\linewidth]{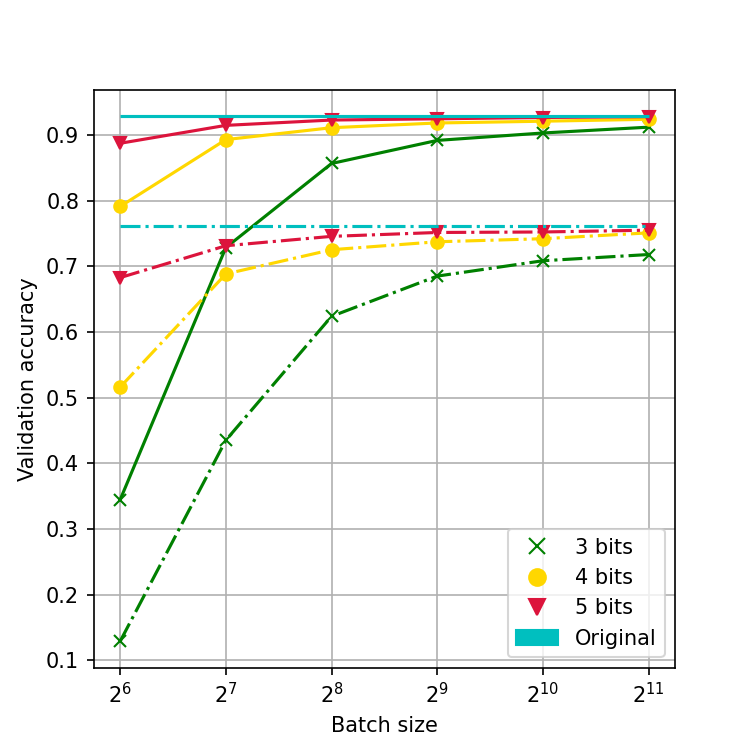}  
\caption{ResNet-50} 
\label{fig:Quant-ResNet50}
\end{subfigure}%
\begin{subfigure}{.33\textwidth}
  \centering
  \includegraphics[width=\linewidth]{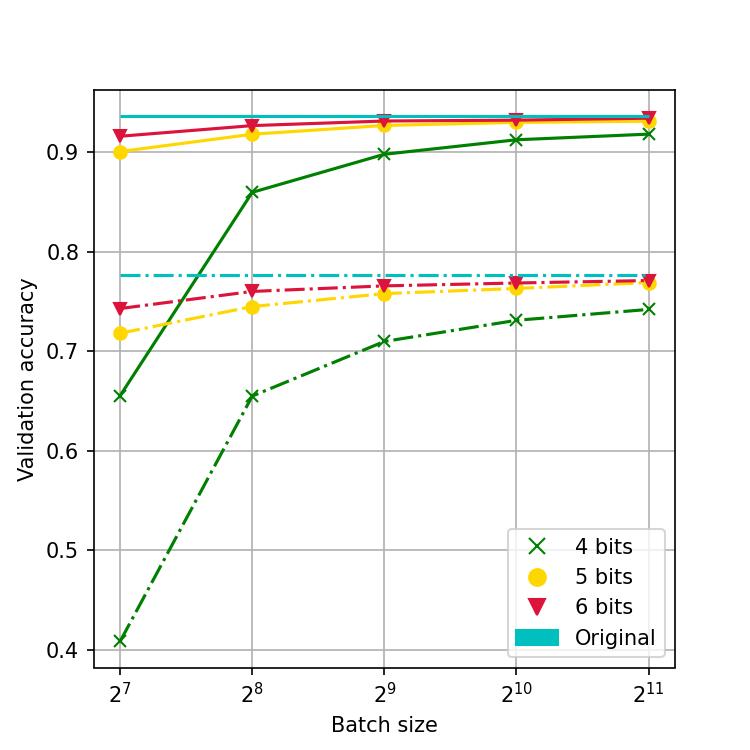}  
\caption{EfficientNet-B1} 
\label{fig:Quant-Efficientnet-B1}
\end{subfigure}%
\caption{Top-1 (dashed lines) and Top-5 (solid lines) accuracy for original and quantized models on ImageNet.}
\label{fig:all_quant_acc}
\end{figure*}

As mentioned in \Cref{sec:CNNs}, we introduce a sampling probability $p\in(0,1]$, associated with  GPFQ  for convolutional layers. This is motivated, in part, by decreasing the computational cost associated with quantizing such layers. Indeed, a batched input tensor of a convolutional layer can be unfolded as a stack of vectorized sliding local blocks, i.e., a matrix. Since, additionally, the kernel can be reshaped into a column vector, matrix-vector multiplication followed by reshaping gives the output of this convolutional layer. On the other hand, due to potentially large overlaps between sliding blocks, the associated matrices have large row size and thus the computational complexity is high. To accelerate our computations, we extract the data used for quantization by setting the stride (which defines the step size of the kernel when sliding through the image) equal to the kernel size and choosing $p=0.25$. This choice gives a good trade-off between accuracy and computational complexity, which both increase with $p$. Recall that the batch size $m\in\mathbb{N}$ denotes the number of samples used for quantizing each layer of a neural network. In all experiments, $b$ is chosen from $\{3,4,5,6\}$.  

\subsection{Results on ImageNet} 
\subsubsection{\textbf{Impact of b and m}}
 The first experiment is designed to explore the effect of the batch size $m$, as well as bit-width $b$, on  the accuracy of the quantized models.  We compute the validation accuracy of quantized networks with respect to different choices of $b$ and $m$. In particular, \cref{table:model-results} shows that, using $b=5$ bits, all quantized models achieve less than $1\%$ loss in top-1 and top-5 accuracy. Moreover, we illustrate the relationship between the quantization accuracy and the batch size $m$ in \cref{fig:all_quant_acc}, where the horizontal lines in cyan, obtained directly from the original validation accuracy of unquantized models, are used for comparison against our quantization method. We observe that (1) all curves with distinct $b$ quickly approach an accuracy ceiling while curves with high $b$ eventually reach a higher ceiling; 
 (2) Quantization with $b\geq 4$ attains near-original model performance with sufficiently large $m$; (3) one can expect to obtain higher quantization accuracy by taking larger $m$ but the extra improvement that results from increasing the batch size rapidly diminishes. 
 
\begin{table*}[ht]
\footnotesize
  \centering 
  \resizebox{0.7\columnwidth}{!}{%
  \begin{threeparttable}
  \caption{ImageNet Top-1 accuracy with weight quantization.}
  \label{table: quant-results-and-comparison}
    \begin{tabular}{c|c|l c c c }
    %{m{15mm} m{70mm} m{18mm}}
    \specialrule{.2em}{.1em}{.1em} 
    Model & Bits & \multicolumn{1}{c}{Method} & Quant Acc (\%) & Ref Acc (\%) & Acc Drop (\%)\\
    \specialrule{.2em}{.1em}{.1em} 
    \multirow{8}{*}{Alexnet} 
    & \multirow{2}{*}{3} & GPFQ (Ours)& 53.22 & 56.52 & 3.30 \\ 
    & & GPFQ (Ours)\tnote{$\dagger$} & 54.77 & 56.52 &  1.75  \\
    \cline{2-6}
     &\multirow{3}{*}{4} & OMSE\cite{choukroun2019low} & 55.52 & 56.62 & 1.10 \\
     &  & GPFQ (Ours) & 55.15 & 56.52 & 1.37  \\
    & & GPFQ (Ours)\tnote{$\dagger$} & 55.51 & 56.52 &  1.01  \\
     \cline{2-6}
     & \multirow{2}{*}{5} & GPFQ (Ours) & 55.67 & 56.52 &  0.85  \\
     & & GPFQ (Ours)\tnote{$\dagger$} & 55.94 & 56.52 &  0.58 \\
     \cline{2-6}
     & \multirow{1}{*}{8} & DoReFa \cite{zhou2016dorefa} & 53.00 & 55.90 &  2.90  \\
    
    \specialrule{.2em}{.1em}{.1em} 
    \multirow{9}{*}{VGG-16} 
    & \multirow{2}{*}{3} & GPFQ (Ours) & 69.67& 71.59 & 1.92  \\ 
    & & GPFQ (Ours)\tnote{$\dagger$} & 70.24 & 71.59&  1.35  \\
    \cline{2-6}
     & \multirow{4}{*}{4} & MSE \cite{banner2018post} & 70.50 &  71.60 & 1.10  \\
     & & OMSE \cite{choukroun2019low} & 71.48 & 73.48 & 2.00 \\
     & & GPFQ (Ours) & 70.70 & 71.59 & 0.89   \\
    & & GPFQ (Ours)\tnote{$\dagger$} & 70.90 & 71.59 &  0.69  \\
     \cline{2-6}
     & \multirow{2}{*}{5} & GPFQ (Ours) & 70.96 & 71.59 &  0.63  \\
     & & GPFQ (Ours)\tnote{$\dagger$} & 71.05 & 71.59 &  0.54  \\
     \cline{2-6}
     & \multirow{1}{*}{8} & Lee et al. \cite{lee2018quantization} & 68.05 & 68.34 & 0.29  \\
    
    \specialrule{.2em}{.1em}{.1em} 
    \multirow{14}{*}{ResNet-18} 
    & \multirow{2}{*}{3} & GPFQ (Ours) & 66.55 & 69.76 & 3.21  \\ 
    & & GPFQ (Ours)\tnote{$\dagger$} & 67.63 & 69.76 &  2.13   \\
    \cline{2-6}
     & \multirow{7}{*}{4} & MSE \cite{banner2018post} & 67.00 &  69.70 & 2.70   \\
     & & OMSE \cite{choukroun2019low} &
     68.38 & 69.64 & 1.26 \\
     & & S-AdaQuant \cite{hubara2020improving} &
     69.40 & 71.97 & 2.57 \\
     & & AdaRound \cite{nagel2020up} &
     68.71 & 69.68 & 0.97 \\
     & & BRECQ \cite{li2021brecq} &
     70.70 & 71.08 & 0.38  \\
     & & GPFQ (Ours) & 68.55 & 69.76 & 1.21   \\
     & & GPFQ (Ours)\tnote{$\dagger$} & 68.81 & 69.76 &  0.95  \\
     \cline{2-6}
     & \multirow{3}{*}{5}  & RQ \cite{louizos2018relaxed} & 65.10 & 69.54 & 4.44  \\ 
     & & GPFQ (Ours) & 69.27 & 69.76 &  0.49   \\
     & & GPFQ (Ours)\tnote{$\dagger$} & 69.50 & 69.76 &  0.26 \\
     \cline{2-6}
     & \multirow{2}{*}{6} & DFQ \cite{nagel2019data} & 66.30 & 70.50 & 4.20  \\ 
     & & RQ \cite{louizos2018relaxed} & 68.65 & 69.54 & 0.89 \\ 
    
    \specialrule{.2em}{.1em}{.1em} 
    \multirow{14}{*}{ResNet-50} 
    & \multirow{2}{*}{3} & GPFQ (Ours) & 71.80 & 76.13 & 4.33  \\ 
    & & GPFQ (Ours)\tnote{$\dagger$}  & 72.18 & 76.13 & 3.95   \\
    \cline{2-6}
     & \multirow{10}{*}{4}  & MSE \cite{banner2018post} &
     73.80 & 76.10 & 2.30 \\
     & & OMSE \cite{choukroun2019low} &
     73.39 & 76.01 & 2.62 \\
     & & OCS + Clip \cite{zhao2019improving} & 69.30 & 76.10 & 6.80 \\
     & & PWLQ \cite{fang2020post} & 73.70 & 76.10 & 2.40 \\
     & & AdaRound \cite{nagel2020up} & 75.23 & 76.07 & 0.84 \\
     & & S-AdaQuant \cite{hubara2020improving} & 75.10 & 77.20 & 2.10 \\
     & & BRECQ \cite{li2021brecq} & 76.29 &  77.00 & 0.71  \\
    %  & & ZeroQ \cite{cai2020zeroq} & 76.08 & 77.72 & 1.64 \\
     & & GPFQ (Ours) & 75.10 & 76.13 & 1.03   \\
    & & GPFQ (Ours)\tnote{$\dagger$}  & 75.30 & 76.13 & 0.83  \\
     \cline{2-6}
     & \multirow{3}{*}{5} & OCS + Clip \cite{zhao2019improving} & 73.40 & 76.10 & 2.70 \\
    & & GPFQ (Ours) & 75.51 & 76.13 & 0.62  \\
     & & GPFQ (Ours)\tnote{$\dagger$} & 75.66 & 76.13 & 0.47  \\
     \cline{2-6}
     & \multirow{1}{*}{8} & IAOI \cite{jacob2018quantization}  & 74.90 & 76.40 & 1.50 \\
    \specialrule{.2em}{.1em}{.1em} 
    \end{tabular}
    \end{threeparttable}%
    }
\end{table*}
 
\subsubsection{\textbf{Comparisons with Baselines}}
Next, we compare GPFQ against other post-training quantization schemes discussed in \Cref{sec:related-work} on various architectures. We note, however, that for a fixed architecture each post-training quantization method starts with a potentially different set of parameters (weights and biases), and these parameters are not available to us. As such, we simply report other methods' accuracies as they appear in their associated papers. Due to this, a perfect comparison between methods is not possible. Another factor that impacts the comparison is that following DoReFa-Net \cite{zhou2016dorefa}, many baseline quantization schemes \cite{zhao2019improving, hubara2020improving, li2021brecq} leave the first and the last layers of DNNs unquantized to alleviate accuracy degradation. On the other hand, we quantize \emph{all} layers of the model. \cref{table: quant-results-and-comparison} displays the number of bits and the method used to quantize each network. It also contains the accuracy of quantized and full-precision models respectively, as well as their difference, i.e. accuracy drop. We report the results of GPFQ (without the $\dagger$ superscript) for all models with $b=3,4,5$. The important observation here is that our method is competitive across architectures and bit-widths, and shows the best performance on a number of them. 

\subsubsection{\textbf{Further Improvement of GPFQ}}
In this section, we show that the validation accuracy of the proposed approach can be further improved by incorporating the following modifications used by prior work: (1) mixing precision for quantization, such as using different bit-widths to quantize fully-connected and convolutional layers respectively \cite{cai2020zeroq} or leaving the last fully-connected layer unquantized \cite{zhou2016dorefa}; (2) applying bias correction \cite{banner2018post, nagel2019data} to the last layer, that is, subtracting the average quantization error from the layer's bias term. In \cref{table: quant-results-and-comparison}, we examine some of these empirical rules by leaving the last layer intact and performing bias correction to remove the noise due to quantization. This variant of GPFQ is highlighted by a $\dagger$ symbol. By using the enhanced GPFQ, the average increment of accuracy exceeds $0.2\%$ for $b=4,5$ bits, and is greater than $0.7\%$ for $b=3$ bits. This demonstrates, empirically, that GPFQ can be easily adapted to incorporate heuristic modifications that improve performance. 

\begin{figure*}[ht]
\mbox{}\hfill 
\begin{subfigure}{.42\textwidth}
  \centering
  \includegraphics[width=\linewidth]{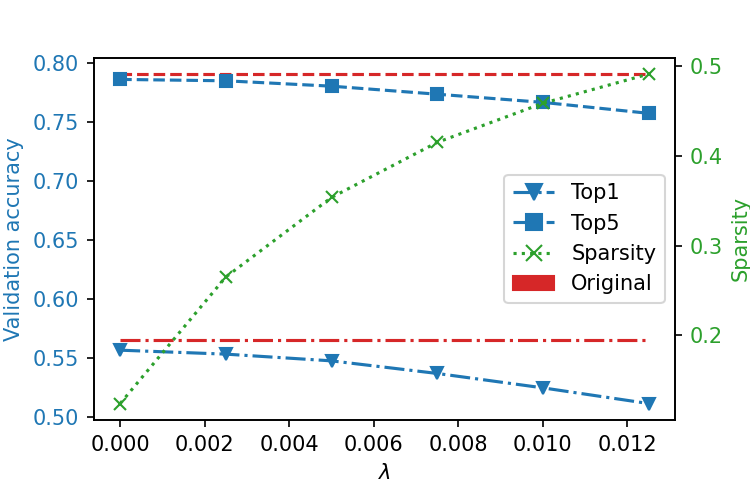}  
\caption{AlexNet with \eqref{eq-update-rule-soft}} 
\label{fig:Quant-AlexNet-sparse-L1}
\end{subfigure}\hfill
\begin{subfigure}{.42\textwidth}
  \centering
  \includegraphics[width=\linewidth]{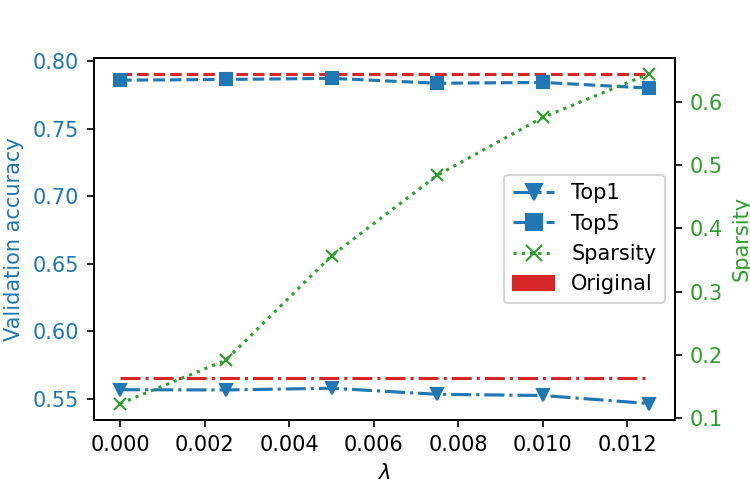}  
\caption{AlexNet with \eqref{eq-update-rule-hard}}
\label{fig:Quant-AlexNet-sparse-L0}
\end{subfigure}\hfill\mbox{}\\
\mbox{}\hfill 
\begin{subfigure}{.42\textwidth}
  \centering
  \includegraphics[width=\linewidth]{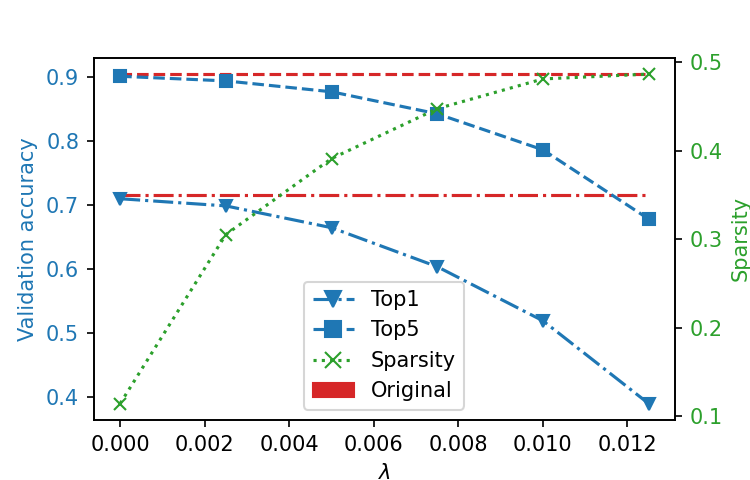}  
\caption{VGG-16 with \eqref{eq-update-rule-soft}} 
\label{fig:Quant-VGG16-sparse-L1}
\end{subfigure}\hfill
\begin{subfigure}{.42\textwidth}
  \centering
  \includegraphics[width=\linewidth]{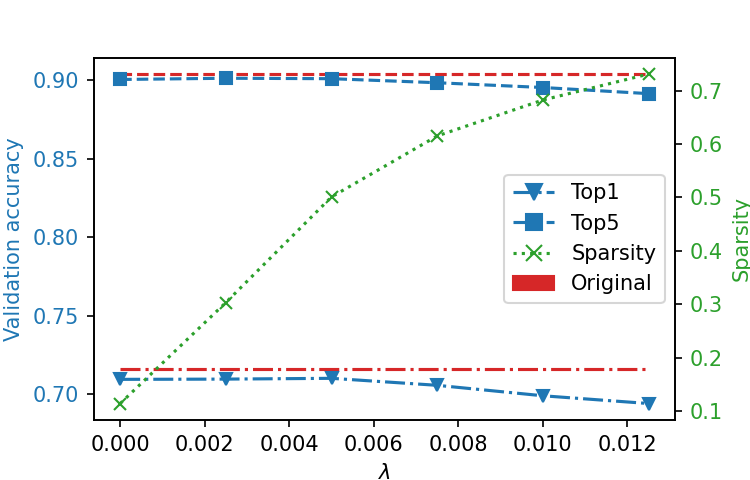}  
\caption{VGG-16 with \eqref{eq-update-rule-hard}} 
\label{fig:Quant-VGG16-sparse-L0}
\end{subfigure}\hfill\mbox{}\\
\mbox{}\hfill 
\begin{subfigure}{.42\textwidth}
  \centering
  \includegraphics[width=\linewidth]{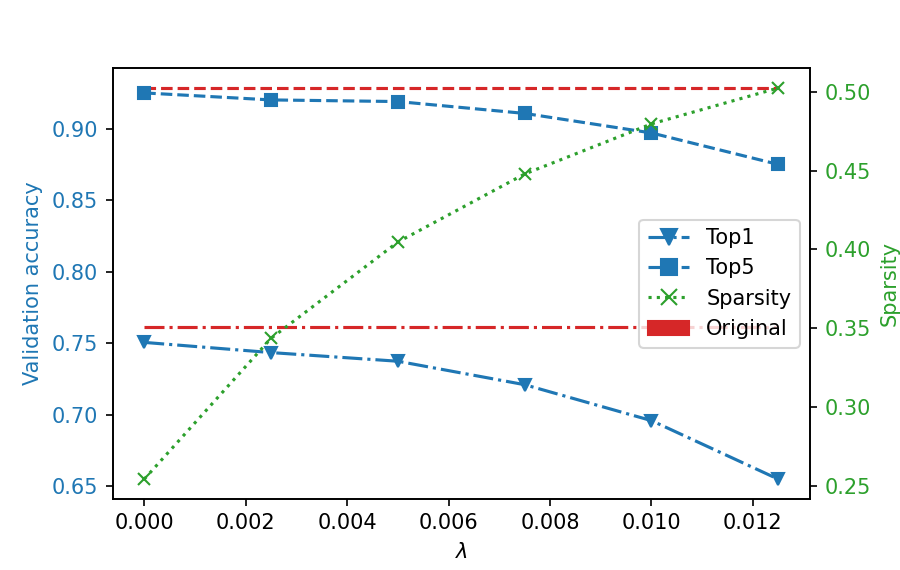}  
\caption{ResNet-50 with \eqref{eq-update-rule-soft}} 
\label{fig:Quant-ResNet50-sparse-L1}
\end{subfigure}\hfill 
\begin{subfigure}{.42\textwidth}
  \centering
  \includegraphics[width=\linewidth]{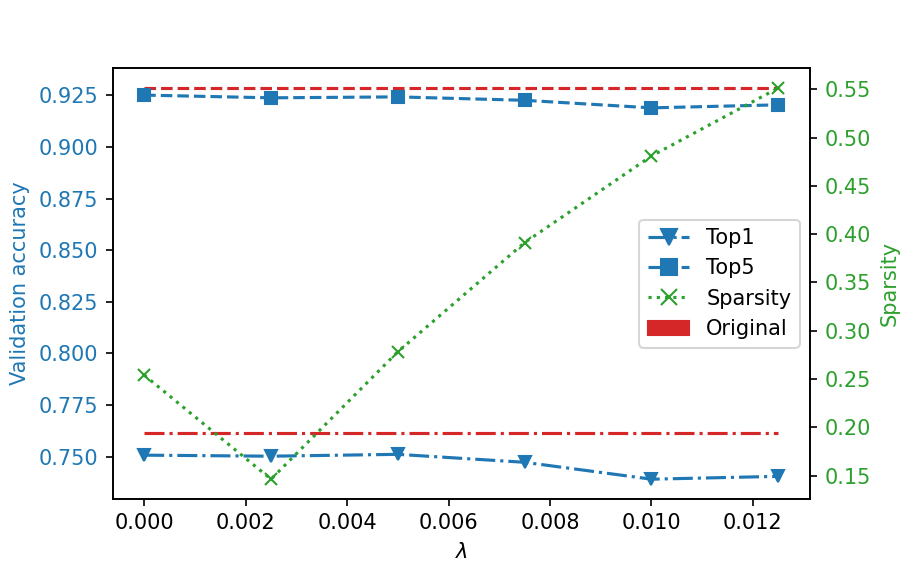}  
\caption{ResNet-50 with \eqref{eq-update-rule-hard}} 
\label{fig:Quant-ResNet50-sparse-L0}
\end{subfigure} 
\hfill\mbox{}
\caption{(1) Left $y$-axis: Top-1 (dashed-dotted lines) and Top-5 (dash lines) accuracy for original (in red) and quantized (in blue) models on ImageNet. (2) Right $y$-axis: The sparsity of quantized models plotted by dotted green lines.}
\label{fig:all_quant_acc_sparse}
\end{figure*}

\subsubsection{\textbf{Sparse Quantization}}\label{sec:exp-sparse}
For our final experiment, we illustrate the effects of sparsity via the sparse quantization introduced in \Cref{sec:sparse-GPFQ}. Recall that the sparse GPFQ with soft thresholding in \eqref{eq-update-rule-soft} uses alphabets $\A_K^\delta$ as in \eqref{eq-alphabet-midtread} while the version of hard thresholding, see \eqref{eq-update-rule-hard}, relies on alphabets $\A_K^{\delta, \lambda}$ as in \cref{eq-alphabet-midtread-variant}. In the setting of our experiment, both $K$ and $\delta$ are still defined and computed as in \Cref{sec:exp-setup}, where the number of bits $b=5$ and the corresponding scalar $C>0$ and batch size $m\in\mathbb{N}$ for each neural network is provided by \cref{table:model-results}. Moreover, the \emph{sparsity} of a given neural network is defined as the proportion of zeros in the weights. According to \cref{eq-update-rule-soft} and \cref{eq-update-rule-hard}, in general, the sparsity of DNNs is boosted as $\lambda$ increases. Hence, we treat $\lambda>0$ as a variable to control sparsity and explore its impact on validation accuracy of different DNNs. As shown in \cref{fig:all_quant_acc_sparse}, we quantize AlexNet, VGG-16, and ResNet-50 using both \eqref{eq-update-rule-soft} and \eqref{eq-update-rule-hard}, with $\lambda\in \{0, 0.0025, 0.005, 0.0075, 0.01, 0.0125 \}$. Curves for validation accuracy and sparsity are plotted against  $\lambda$. We note that, for all tested models, sparse GPFQ with hard thresholding, i.e. \eqref{eq-update-rule-hard}, outperforms soft thresholding, achieving significantly higher sparsity and better accuracy. For example, by quantizing AlexNet and VGG-16 with \eqref{eq-update-rule-hard}, one can maintain near-original model accuracy when half the weights are quantized to zero, which implies a remarkable compression rate $\frac{0.5b}{32}=\frac{2.5}{32}\approx  7.8\%$. Similarly, \cref{fig:Quant-ResNet50-sparse-L0} and \cref{fig:Quant-ResNet50-sparse-L1} show that ResNet-50 can attain $40\%$ sparsity with subtle decrement in accuracy. Additionally, in all cases, one can expect to get higher sparsity by increasing $\lambda$ while the validation accuracy tends to drop gracefully. Moreover, in \cref{fig:Quant-ResNet50-sparse-L1}, we observe that the sparsity of quantized ResNet50 with $\lambda=0.0025$ is even lower than the result when thresholding functions are not used, that is, $\lambda=0$. A possible reason is given as follows. In contrast with $\A_K^\delta$, the alphabet $\A_K^{\delta,\lambda}$ has only one element $0$ between $-\lambda$ and $\lambda$. Thus, to compensate for the lack of small alphabet elements and also reduce the path following error, sparse GPFQ in \eqref{eq-update-rule-hard} converts more weights to nonzero entries of $\A_K^{\delta,\lambda}$, which in turn dampens the upward trend in sparsity.

\section*{Acknowledgements}
This work was supported in part by National Science Foundation Grant DMS-2012546. The authors thank Eric Lybrand for stimulating discussions on the topics of this paper.

\bibliographystyle{abbrvnat}
\bibliography{citations}

\newpage 
\appendix
\section{Useful Lemmata}
\begin{lemma}\label{lemma-update-rule-dynamic}
In the context of \eqref{eq-update-rule-dynamic}, we have $q_t = \mathcal{Q}\biggl(\frac{\langle \widetilde{X}_t^{(i-1)}, u_{t-1}+w_t X_t^{(i-1)} \rangle}{\|\widetilde{X}^{(i-1)}_t\|_2^2} \biggr)$. Here, we suppose $\widetilde{X}_t^{(i-1)}\neq 0$.
\end{lemma}
\begin{proof}

According to \eqref{eq-update-rule-dynamic}, $q_t = \arg\min\limits_{p\in\mathcal{A}}\bigl\| u_{t-1} +w_t X_t^{(i-1)}-p\widetilde{X}^{(i-1)}_t\bigr\|_2^2$. Expanding the square and removing the terms irrelevant to $p$, we obtain
\begin{align*}
q_t &= \arg\min\limits_{p\in\mathcal{A}} \Bigl( p^2 \|\widetilde{X}_t^{(i-1)}\|_2^2 - 2p\langle \widetilde{X}_t^{(i-1)}, u_{t-1}+w_t X_t^{(i-1)}\rangle   \Bigr)\\
&=\arg\min\limits_{p\in\mathcal{A}}\biggl( p^2 - 2p\cdot \frac{\langle  \widetilde{X}_t^{(i-1)},u_{t-1}+w_t X_t^{(i-1)}  \rangle}{\|\widetilde{X}_t^{(i-1)}\|_2^2}\biggr) \\
&= \arg\min\limits_{p\in\mathcal{A}}\biggl( p-\frac{\langle  \widetilde{X}_t^{(i-1)},u_{t-1}+w_t X_t^{(i-1)}  \rangle}{\|\widetilde{X}_t^{(i-1)}\|_2^2}\biggr)^2 \\
&= \arg\min\limits_{p\in\mathcal{A}}\biggl| p-\frac{\langle  \widetilde{X}_t^{(i-1)},u_{t-1}+w_t X_t^{(i-1)}  \rangle}{\|\widetilde{X}_t^{(i-1)}\|_2^2}\biggr| \\
&= \mathcal{Q}\biggl(\frac{\langle \widetilde{X}_t^{(i-1)}, u_{t-1}+w_t X_t^{(i-1)} \rangle}{\|\widetilde{X}^{(i-1)}_t\|_2^2} \biggr).
\end{align*}
In the last equality, we used the definition of \eqref{eq:MSQ-midtread}.
\end{proof}

\begin{lemma}\label{lemma:update-rule-optimization}
Suppose $\widetilde{X}_t^{(i-1)}\neq 0$. The closed-form expression of $q_t$ in \eqref{eq:update-rule-optimization} is given by 
$q_t = \mathcal{Q}\circ s_\lambda\biggl( \frac{\langle \widetilde{X}_t^{(i-1)}, u_{t-1}+w_t X_t^{(i-1)} \rangle}{\|\widetilde{X}^{(i-1)}_t\|_2^2}\biggr)$.
Here, $s_\lambda(x):=\sign(x)\max\{|x|-\lambda,0\}$ is the soft thresholding function.
\end{lemma}
\begin{proof}
Expanding the square and removing the terms irrelevant to $p$, we obtain
\begin{align}\label{eq:update-rule-optimization-eq1}
q_t &= \arg\min\limits_{p\in\A} \Bigl( \frac{p^2}{2} \|\widetilde{X}_t^{(i-1)}\|_2^2 - p\langle \widetilde{X}_t^{(i-1)}, u_{t-1}+w_t X_t^{(i-1)}\rangle+\lambda |p|\|\widetilde{X}_t^{(i-1)}\|_2^2  \Bigr) \nonumber \\
&=\arg\min\limits_{p\in\A}\biggl( \frac{p^2}{2} - p\cdot \frac{\langle  \widetilde{X}_t^{(i-1)},u_{t-1}+w_t X_t^{(i-1)}  \rangle}{\|\widetilde{X}_t^{(i-1)}\|_2^2} + \lambda |p|\biggr) \nonumber \\
&=\arg\min\limits_{p\in\A}\biggl( \frac{p^2}{2}-\alpha_t p + \lambda |p|\biggr) 
\end{align}
where 
$\alpha_t:= \frac{\langle \widetilde{X}_t^{(i-1)}, u_{t-1}+w_t X_t^{(i-1)} \rangle}{\|\widetilde{X}^{(i-1)}_t\|_2^2}$. Define $g_t(p):=\frac{1}{2}p^2-\alpha_t p+\lambda |p|$ for $p\in\R$. By \eqref{eq-alphabet-midtread}, we have $q_t=\arg\min_{p\in\A} g_t(p)=\arg\min_{\substack{ |k|\leq K \\ k\in\mathbb{Z}}} g_t(k\delta)$.
Now we analyze two cases $\alpha_t\geq 0$ and $\alpha_t<0$. The idea is to investigate the behaviour of $g_t(k\delta)$ over $k\in \{-K,...,K\}$.

\noindent (I) Assume $\alpha_t\geq 0$. Since $g_t(k\delta)>g_t(0)=0$ for all $-K\leq k \leq -1$, then $g_t(k\delta)$ is minimized at some $k\geq 0$. Note that $g_t(p)$ is a convex function passing through the origin. So, for $1\leq k\leq K-1$, $g_t(k\delta)$ is the minimum if and only if 
$g_t(k\delta) \leq \min\{g_t((k+1)\delta), g_t((k-1)\delta)\}$.

It is easy to verify that the condition above is equivalent to 
\begin{equation}\label{eq:update-rule-optimization-eq2}
\Bigl(k-\frac{1}{2} \Bigr)\delta+\lambda  \leq \alpha_t\leq \Bigl(k+\frac{1}{2} \Bigr)\delta +\lambda.
\end{equation}
It only remains to check $k=0$ and $k=K$. For $k=0$, note that when $\alpha_t\in [0, \delta/2+\lambda]$, we have 
\begin{equation}\label{eq:update-rule-optimization-eq3}
    g_t(\delta)\geq g_t(0)=0,
\end{equation}
and if $\alpha_t\geq (K-\frac{1}{2})\delta+\lambda$, then 
\begin{equation}\label{eq:update-rule-optimization-eq4}
    g_t(K\delta)\leq g_t((K-1)\delta).
\end{equation}
Combining \eqref{eq:update-rule-optimization-eq2}, \eqref{eq:update-rule-optimization-eq3}, and \eqref{eq:update-rule-optimization-eq4}, we conclude that 
\begin{equation}\label{eq:update-rule-optimization-eq5}
q_t=\arg\min_{\substack{ |k|\leq K \\ k\in\mathbb{Z}}} g_t(k\delta)=
\begin{cases}
0 & \text{if}\ 0\leq \alpha_t< \frac{\delta}{2}+\lambda, \\
k\delta &\text{if} \ |\alpha_t - \lambda - k\delta|\leq \frac{\delta}{2}\ \text{and} \ 1\leq k\leq K-1, \\
K\delta &\text{if} \ \alpha_t\geq \lambda+\frac{\delta}{2}+(K-1)\delta.
\end{cases}
\end{equation}
\noindent (II) In the opposite case where $\alpha_t<0$, it suffices to minimize $g_t(k\delta)$ with $k\leq 0$ because $g_t(k\delta)>0$ for all $k\geq 1$. Again, notice that $g_t(p)$ is a convex function on $[-\infty, 0]$ satisfying $g_t(0)=0$. Applying a similar argument as in the case $\alpha_t\geq 0$, one can get 
\begin{equation}\label{eq:update-rule-optimization-eq6}
q_t=\arg\min_{\substack{ |k|\leq K \\ k\in\mathbb{Z}}} g_t(k\delta)=
\begin{cases}
0 & \text{if}\  -\frac{\delta}{2}-\lambda< \alpha_t<0, \\
k\delta &\text{if} \ |\alpha_t + \lambda - k\delta|\leq \frac{\delta}{2}\ \text{and} \ -(K-1)\leq k\leq -1,\\
-K\delta &\text{if} \ \alpha_t \leq  -\lambda -\frac{\delta}{2}-(K-1)\delta.
\end{cases}
\end{equation}
It follows from \eqref{eq:update-rule-optimization-eq5} and \eqref{eq:update-rule-optimization-eq6} that $ q_t = \mathcal{Q}(s_\lambda(\alpha_t))=\mathcal{Q}\circ s_\lambda\biggl( \frac{\langle \widetilde{X}_t^{(i-1)}, u_{t-1}+w_t X_t^{(i-1)} \rangle}{\|\widetilde{X}^{(i-1)}_t\|_2^2}\biggr) $
where $s_\lambda(x):=\sign(x)\max\{|x|-\lambda,0\}$ is the soft thresholding function.
\end{proof}

\begin{lemma}\label{lemma-uniform-square}
Let $\mathrm{Unif}(B_r)$ denote the uniform distribution on the closed ball $B_r\subset\R^m$ with center at the origin and radius $r>0$. Suppose that the random vector $X\in \R^m$ is drawn from $\mathrm{Unif}(B_r)$. Then we have
$\E\|X\|_2^2=\frac{mr^2}{m+2}$.
\end{lemma} 
\begin{proof} 
Note that the density function of $\mathrm{Unif}(B_r)$ is given by $f(x)=\frac{1}{\mathrm{vol}(B_r)}\mathbbm{1}_{B_r}(x)$ where $\mathrm{vol}(B_r)=r^m\pi^{\frac{m}{2}}/\Gamma(\frac{m}{2}+1)$ is the volume of $B_r$.  Moreover, by integration in spherical coordinates, one can get
\begin{align*}
    &\E \|X\|_2^2 = \int_{\R^m} \|x\|_2^2 f(x)\, dx = \int_0^\infty \int_{\mathbb{S}^{m-1}} z^{m-1} \|zx\|_2^2  f(zx)\, d\sigma(x)\, dz  \\
    &= \int_0^r \int_{\mathbb{S}^{m-1}} \frac{z^{m+1}}{\mathrm{vol}(B_r)} \, d\sigma(x)\, dz  = \frac{\sigma(\mathbb{S}^{m-1})}{\mathrm{vol}(B_r)} \int_0^r z^{m+1} \, dz 
    = \frac{mr^2}{m+2}.
\end{align*}
Here, $\sigma(\mathbb{S}^{m-1})=2\pi^{\frac{m}{2}}/\Gamma(\frac{m}{2})$ is the spherical measure (area) of the unit sphere $\mathbb{S}^{m-1}\subset\R^m$. 
\end{proof}

\textbf{Orthogonal Projections.} Given a closed subspace $S \subseteq \mathbb{R}^m$, we denote the orthogonal projection onto $S$ by $P_S$. In particular, if $z\in\mathbb{R}^m$ is a vector, then we use $P_z$ and $P_{z^\perp}$ to represent orthogonal projections onto $\mathrm{span}(z)$ and $\mathrm{span}(z)^\perp$ respectively. Hence, for any $x\in\mathbb{R}^m$, we have 
\begin{equation}\label{eq-orth-proj}
P_z(x)=\frac{\langle z, x\rangle z}{\|z\|_2^2}, \quad x = P_z(x) + P_{z^\perp}(x), \quad\text{and}\quad \|x\|_2^2 = \|P_z(x)\|_2^2 + \|P_{z^\perp}(x)\|_2^2.
\end{equation}      

\begin{lemma}\label{lemma-diff}
Let $\mathcal{A}$ be as in \eqref{eq-alphabet-midtread} with step size $\delta>0$, and largest element $q_{\text{max}}$. Suppose that $w\in\R^{N_0}$ satisfies $\|w\|_\infty\leq q_{\mathrm{max}}$, and consider the quantization scheme given by \eqref{eq-update-rule-single}. Let $\theta_t:=\angle(X_t, u_{t-1})$ be the angle between $X_t$ and $u_{t-1}$. Then, for $t=1,2,\ldots,N_0$, we have 
\begin{equation}\label{lemma-diff-eq}
\|u_t\|_2^2-\|u_{t-1}\|_2^2\leq 
\begin{cases}
\frac{\delta^2}{4}\|X_t\|_2^2-\|u_{t-1}\|_2^2\cos^2\theta_t &\text{if}\; \Bigl| w_t+\frac{\|u_{t-1}\|_2}{\|X_t\|_2}\cos\theta_t \Bigr|\leq q_{\text{max}}, \\
0 &\text{otherwise}.
\end{cases}  
\end{equation}
\end{lemma}
\begin{proof}
By applying \eqref{eq-orth-proj} and \eqref{eq-update-rule-single}, we get
\begin{align}\label{lemma-diff-eq1}
\|P_{X_t}(u_t)\|_2^2 &= \frac{(X_t^\top u_t)^2}{\|X_t\|_2^2} =\frac{(X_t^\top u_{t-1}+(w_t-q_t)\|X_t\|_2^2)^2}{\|X_t\|_2^4}\|X_t\|_2^2 \nonumber\\
&= \Bigl( w_t + \frac{X_t^\top u_{t-1}}{\|X_t\|_2^2}-q_t\Bigr)^2\|X_t\|_2^2 
= \Bigl( w_t+\frac{\|u_{t-1}\|_2}{\|X_t\|_2}\cos\theta_t -q_t\Bigr)^2\|X_t\|_2^2.
\end{align}
The last equation holds because $X_t^\top u_{t-1}=\|X_t\|_2\|u_{t-1}\|_2\cos\theta_t$. Note that 
\[
\Bigl( w_t+\frac{\|u_{t-1}\|_2}{\|X_t\|_2}\cos\theta_t -q_t\Bigr)^2-\Bigl(\frac{\|u_{t-1}\|_2}{\|X_t\|_2}\cos\theta_t \Bigr)^2 =
\Bigl(\underbrace{ w_t+\frac{2\|u_{t-1}\|_2}{\|X_t\|_2}\cos\theta_t -q_t}_{\text{(I)}}\Bigr)(\underbrace{w_t-q_t}_{\text{(II)}}),
\]
$|w_t|\leq q_{\text{max}}$, and $q_t=\mathcal{Q}\Bigl(w_t+\frac{\|u_{t-1}\|_2}{\|X_t\|_2}\cos\theta_t \Bigr)$. If $\Bigl( w_t+\frac{\|u_{t-1}\|_2}{\|X_t\|_2}\cos\theta_t \Bigr)> q_{\text{max}}$, then $q_t=q_{\text{max}}$ and thus $0\leq q_t-w_t\leq \frac{\|u_{t-1}\|_2}{\|X_t\|_2}\cos\theta_t$. So 
$\text{(I)}\geq w_t+2(q_t-w_t)-q_t=q_t-w_t\geq 0$
and $\text{(II)}\leq 0$. Moreover, if $\Bigl( w_t+\frac{\|u_{t-1}\|_2}{\|X_t\|_2}\cos\theta_t \Bigr)<- q_{\text{max}}$, then $q_t=-q_{\text{max}}$ and $\frac{\|u_{t-1}\|_2}{\|X_t\|_2}\cos\theta_t \leq q_t-w_t\leq 0$. Hence, $\text{(I)}\leq w_t+2(q_t-w_t)-q_t=q_t-w_t\leq 0$ and $\text{(II)}\geq 0$. It follows that 
\begin{equation}\label{lemma-diff-eq3}
\Bigl( w_t+\frac{\|u_{t-1}\|_2}{\|X_t\|_2}\cos\theta_t  -q_t\Bigr)^2\leq\Bigl(\frac{\|u_{t-1}\|_2}{\|X_t\|_2}\cos\theta_t \Bigr)^2 
\end{equation}
when $\Bigl| w_t+\frac{\|u_{t-1}\|_2}{\|X_t\|_2}\cos\theta_t \Bigr|> q_{\text{max}}$. Now, assume that $\Bigl| w_t+\frac{\|u_{t-1}\|_2}{\|X_t\|_2}\cos\theta_t \Bigr|\leq q_{\text{max}}$. In this case, since the argument of $\mathcal{Q}$ lies in the active range of $\mathcal{A}$, we obtain 
\begin{equation}\label{lemma-diff-eq4}
\Bigl( w_t+\frac{\|u_{t-1}\|_2}{\|X_t\|_2}\cos\theta_t -q_t\Bigr)^2\leq \frac{\delta^2}{4}.
\end{equation}
Applying \eqref{lemma-diff-eq3} and \eqref{lemma-diff-eq4} to \eqref{lemma-diff-eq1}, one can get 
\begin{equation}\label{lemma-diff-eq5}
    \|P_{X_t}(u_t)\|_2^2\leq 
\begin{cases}
\frac{\delta^2}{4}\|X_t\|_2^2 &\text{if}\; \Bigl| w_t+\frac{\|u_{t-1}\|_2}{\|X_t\|_2}\cos\theta_t \Bigr|\leq q_{\text{max}}, \\
 \|u_{t-1}\|_2^2\cos^2\theta_t  &\text{otherwise}.
\end{cases}
\end{equation}
Further, we have
\begin{equation}\label{lemma-diff-eq2}
P_{X_t^\perp}(u_t) = P_{X_t^\perp}(u_{t-1}+w_tX_t-q_tX_t)=P_{X_t^\perp}(u_{t-1}).
\end{equation}
It follows that 
\begin{align*}
\|u_t\|_2^2 -\|u_{t-1}\|_2^2 &= \|P_{X_t}(u_t)\|_2^2+\|P_{X_t^\perp}(u_t)\|_2^2-\|u_{t-1}\|_2^2 \\
&=\|P_{X_t}(u_t)\|_2^2 + \|P_{X_t^\perp}(u_{t-1})\|_2^2-\|u_{t-1}\|_2^2  &\text{(by \eqref{lemma-diff-eq2})} \\
&= \|P_{X_t}(u_t)\|_2^2-\|P_{X_t}(u_{t-1})\|_2^2 &\text{(using \eqref{eq-orth-proj})}\\
&= \|P_{X_t}(u_t)\|_2^2-\|u_{t-1}\|_2^2\cos^2\theta_t.
\end{align*}
Substituting $\|P_{X_t}(u_t)\|_2^2$ with its upper bounds in \eqref{lemma-diff-eq5}, we obtain \eqref{lemma-diff-eq}.
\end{proof}

\begin{lemma}\label{lemma-bound}
Let $\mathcal{A}$ be as in \eqref{eq-alphabet-midtread} with step size $\delta>0$, and  largest element $q_{\max}$. Suppose that $w\in\R^{N_0}$  satisfies $\|w\|_\infty\leq q_{\mathrm{max}}$, and consider the quantization scheme  given by \eqref{eq-update-rule-single}. Additionally, denote the information of the first $t-1$ quantization steps by a $\sigma$-algebra $\mathcal{F}_{t-1}$, and let $\beta, \eta>0$, $s^2\in (0,1)$. Then the following results hold for $t=1,2,\ldots,N_0$.
\begin{enumerate} 
    \item 
    $\E e^{\eta\|u_{t}\|_2^2} \leq \max\Bigl\{
    \E(e^{\frac{\eta\delta^2}{4}\|X_t\|_2^2}e^{\eta\|u_{t-1}\|_2^2(1-\cos^2\theta_t)}),
    \E e^{\eta\|u_{t-1}\|_2^2} \Bigr\} $.
    \item $
    \E(e^{\eta\beta\|u_{t-1}\|_2^2(1-\cos^2\theta_t)}\mid \mathcal{F}_{t-1}) \leq 
    -\E(\cos^2\theta_t\mid\mathcal{F}_{t-1})(e^{\eta\beta\|u_{t-1}\|_2^2}-1)+e^{\eta\beta\|u_{t-1}\|_2^2} $.
\end{enumerate}
Here, $\theta_t$ is the angle between $X_t$ and $u_{t-1}$.
\end{lemma}
\begin{proof}
\noindent (1) In the $t$-th step, by Lemma~\ref{lemma-diff}, we have
\[
\|u_t\|_2^2-\|u_{t-1}\|_2^2\leq 
\begin{cases}
\frac{\delta^2}{4}\|X_t\|_2^2-\|u_{t-1}\|_2^2\cos^2\theta_t &\text{if}\; \Bigl| w_t+\frac{\|u_{t-1}\|_2}{\|X_t\|_2}\cos\theta_t \Bigr|\leq q_{\text{max}}, \\
        0 &\text{otherwise},
\end{cases}  
\]
where $\theta_t=\angle (X_t, u_{t-1})$ is the angle between $X_t$ and $u_{t-1}$. On the one hand, if $\Bigl| w_t+\frac{\|u_{t-1}\|_2}{\|X_t\|_2}\cos\theta_t \Bigr|\leq q_{\text{max}}$, we obtain
\begin{equation}\label{lemma-bound-eq2}
\E e^{\eta\|u_t\|_2^2} = \E(e^{\eta(\|u_{t}\|_2^2-\|u_{t-1}\|_2^2)}e^{\eta\|u_{t-1}\|_2^2}) 
\leq \E(e^{\frac{\eta\delta^2}{4}\|X_t\|_2^2}e^{\eta\|u_{t-1}\|_2^2(1-\cos^2\theta_t)} )
\end{equation}
On the other hand, if $\Bigl| w_t+\frac{\|u_{t-1}\|_2}{\|X_t\|_2}\cos\theta_t \Bigr|> q_{\text{max}}$, we get
\begin{equation}\label{lemma-bound-eq3} 
\E e^{\eta\|u_t\|_2^2}=\E(e^{\eta(\|u_t\|_2^2-\|u_{t-1}\|_2^2)}e^{\eta\|u_{t-1}\|_2^2})\leq \E e^{\eta\|u_{t-1}\|_2^2}.
\end{equation}
Combining \eqref{lemma-bound-eq2} and \eqref{lemma-bound-eq3}, we conclude that
\[
\E e^{\eta\|u_{t}\|_2^2} \leq \max\Bigl\{
    \E(e^{\frac{\eta\delta^2}{4}\|X_t\|_2^2}e^{\eta\|u_{t-1}\|_2^2(1-\cos^2\theta_t)}),
    \E e^{\eta\|u_{t-1}\|_2^2} \Bigr\}.
\]

\noindent (2) Conditioning on $\mathcal{F}_{t-1}$, the function $f(x)=e^{\eta\beta x\|u_{t-1}\|_2^2}$ is convex. It follows that
\begin{align*}
\E(e^{\eta\beta\|u_{t-1}\|_2^2(1-\cos^2\theta_t)}\mid \mathcal{F}_{t-1}) &=\E(f(\cos^2\theta_t\cdot 0+(1-\cos^2\theta_t)\cdot 1)\mid \mathcal{F}_{t-1})  \\ 
&\leq \E( \cos^2\theta_t+(1-\cos^2\theta_t)e^{\eta\beta\|u_{t-1}\|_2^2} \mid \mathcal{F}_{t-1}) \\
&\leq \E(\cos^2\theta_t\mid \mathcal{F}_{t-1})+(1-\E(\cos^2\theta_t\mid\mathcal{F}_{t-1}))e^{\eta\beta\|u_{t-1}\|_2^2} \\
&= -\E(\cos^2\theta_t\mid\mathcal{F}_{t-1})(e^{\eta\beta\|u_{t-1}\|_2^2}-1)+e^{\eta\beta\|u_{t-1}\|_2^2}.
\end{align*}
\end{proof}

\section{Fusing Convolution and Batch Normalization Layers}
\label{sec:fusion-bn}
For many  neural networks, e.g. MobileNets and ResNets, a convolutional layer is usually followed by a batch normalization (BN) layer to normalize the output. Here, we show how our quantization approach admits a simple modification that takes into account such BN layers. Specifically, denote the convolution operator by * and suppose that a convolutional layer 
\begin{equation}\label{eq:conv2d}
    f_\mathrm{conv}(x):= w_\mathrm{conv} * x + b_\mathrm{conv}
\end{equation}
is followed by a BN layer given by 
\begin{equation}\label{eq:batchnorm2d}
f_\mathrm{bn} (x) := \frac{x-\hat{\mu}}{\sqrt{\hat{\sigma}^2 + \epsilon}} \cdot w_\mathrm{bn} + b_\mathrm{bn}.
\end{equation}
Here, $w_\mathrm{conv}$, $w_\mathrm{bn}$, $b_\mathrm{conv}$, and $b_\mathrm{bn}$ are learned parameters and $\hat{\mu}$, $\hat{\sigma}$ are the running mean and standard-deviation respectively while $\epsilon>0$ is to keep the denominator bounded away from 0. Note that the parameters in both \cref{eq:conv2d} and \cref{eq:batchnorm2d} are calculated per-channel over the mini-batches during training, but fixed thereafter. 

\begin{table}[ht]
\caption{Top-1 accuracy drop for ResNet-18 and ResNet-50.}
\label{table:fusion-resnet}
\centering 
\begin{tabular}{lcc|cc|cc}
\toprule
\multirow{2}{*}{Model} & \multirow{2}{*}{$b$} & \multirow{2}{*}{$m$} & \multicolumn{2}{c|}{Unfused} & \multicolumn{2}{c}{Fused}\\ \cline{4-7} 
& & & C & Acc Drop (\%) & C & Acc Drop (\%) \\
\midrule 
\multirow{4}{*}{ResNet-18} &4 & 2048 & \multirow{4}{*}{1.16} & 1.63 & \multirow{4}{*}{1.29} & 1.72 \\
&4 & 4096 & &  1.21 & & 1.18 \\
&5 & 2048 & &  0.71 & & 0.72 \\
&5 & 4096 & &  0.49 & & 0.51 \\
\midrule 
\multirow{3}{*}{ResNet-50}
& 5 & 512 & \multirow{3}{*}{1.81} & 0.97 & \multirow{3}{*}{1.82} & 1.03 \\
& 5 & 1024 & & 0.90 &  & 0.81\\
& 5 & 2048 & & 0.62 & & 0.64 \\
\bottomrule
\end{tabular}
\vskip -0.1in
\end{table}

Thus, to quantize the convolutional and subsequent BN layers simultaneously, we first observe that  we can write
\begin{equation}
\label{eq:fusion}
f_\mathrm{bn}\circ f_\mathrm{conv}(x) = w_\mathrm{new} * x + b_\mathrm{new}
\end{equation}
with 
\[
w_\mathrm{new} := \frac{w_\mathrm{conv} w_\mathrm{bn}}{\sqrt{\hat{\sigma}^2+\epsilon}}, \quad b_\mathrm{new}:= \frac{(b_\mathrm{conv} -\hat{\mu})w_\mathrm{bn}}{\sqrt{\hat{\sigma}^2+\epsilon}} + b_\mathrm{bn}. 
\]

As a result, to quantize the convolutional and subsequent BN layer simulatenously, we can simply quantize the parameters $w_\mathrm{new}, b_\mathrm{new}$ in \eqref{eq:fusion} using our methods. 
Although BN layers are not quantized in our experiments in \Cref{sec:experiments}, we will show here that the proposed algorithm GPFQ is robust to neural network fusion as described above. In \cref{table:fusion-resnet}, we compare the Top-1 quantization accuracy between fused ResNets and unfused ResNets when quantized using our methods with different bits and batch sizes. Note that the scalar $C$ for unfused networks remains the same as in \cref{table:model-results} while $C$ for fused networks is selected using the procedure after \cref{eq:exp-step-size}. We observe that the performance of GPFQ for fused ResNet-18 and ResNet-50 is quite similar to that for unfused networks.

\section{Quantizing Large Weights}
\label{sec:results-large-weights}
% red color begin 

In this section, we demonstrate that the proposed quantization algorithm \eqref{eq-update-rule-single} is still effective for weights with magnitudes that exceed the largest element, $q_{\max}=K\delta$, in the alphabet set $\mathcal{A}$. 

Specifically, we prove \cref{thm:large-weights}, bounding the expected error when $n:=n(\delta)$ entries of $w$ are greater than $K\delta$. In turn, \cref{thm:large-weights} suggests that in some cases, choosing $\delta$ such that $n(\delta)>0$ may be advantageous, a finding that is consistent with our experiments in \Cref{sec:experiments}. We begin with the following lemma needed to prove \cref{thm:large-weights}.

%In this section, we show that the proposed quantization algorithm \eqref{eq-update-rule-single} still works for the weights with magnitude exceeding the boundary of alphabet set $\mathcal{A}$. The following lemma will be used for the main result \cref{thm:large-weights}.

\begin{lemma}\label{lemma:bound-large-weights}
Let $\mathcal{A}$ be as in \eqref{eq-alphabet-midtread} with step size $\delta>0$, and largest element $q_{\text{max}}$. Suppose that $w\in\R^{N_0}$ satisfies $\|w\|_\infty\leq k q_{\mathrm{max}}$ for some $k>1$, and consider the quantization scheme given by \eqref{eq-update-rule-single}. Let $\theta_t:=\angle(X_t, u_{t-1})$ be the angle between $X_t$ and $u_{t-1}$. Then
\begin{equation}\label{eq:bound-large-weights-eq1}
\|u_t\|_2^2 \leq 
\begin{cases}
\frac{\delta^2}{4}\|X_t\|_2^2+\|u_{t-1}\|_2^2(1-\cos^2\theta_t) &\text{if}\; \bigl| w_t+\frac{\|u_{t-1}\|_2}{\|X_t\|_2}\cos\theta_t \bigr|\leq q_{\max}, \\
\|u_{t-1}\|_2^2 &\text{if}\; \bigl| w_t+\frac{\|u_{t-1}\|_2}{\|X_t\|_2}\cos\theta_t \bigr|> q_{\max} \; \text{and} \; |w_t|\leq q_\mathrm{max}, \\
(\|u_{t-1}\|_2 + (k-1)q_\mathrm{max}\|X_t\|_2)^2 &\text{if} \; \bigl| w_t+\frac{\|u_{t-1}\|_2}{\|X_t\|_2}\cos\theta_t \bigr|> q_{\max} \; \text{and} \; |w_t|> q_\mathrm{max}
\end{cases}  
\end{equation}
holds for $t=1,2,\ldots,N_0$.
\end{lemma}

\begin{proof}
The first two cases in \eqref{eq:bound-large-weights-eq1} are covered by \cref{lemma-diff}. So it remains to consider the case where $\bigl| w_t+\frac{\|u_{t-1}\|_2}{\|X_t\|_2}\cos\theta_t \bigr|> q_{\max}$ and $|w_t|> q_\mathrm{max}$. As in the proof of \cref{lemma-diff}, we have 
\[
\|u_t\|_2^2 = (v_t-q_t)^2\|X_t\|_2^2 + (1-\cos^2\theta_t) \|u_{t-1}\|_2^2
\]
where $v_t:= w_t+\frac{\|u_{t-1}\|_2}{\|X_t\|_2}\cos\theta_t$. Since $q_t=\mathcal{Q}(v_t)$ and $|v_t|>q_\mathrm{max}$, we get $q_t=\sign(v_t)q_\mathrm{max}$. It follows that 
\begin{align}\label{eq:bound-large-weights-eq2}
\|u_t\|_2^2 &= (v_t -\sign(v_t)q_\mathrm{max})^2\|X_t\|_2^2 + (1-\cos^2\theta_t) \|u_{t-1}\|_2^2 \nonumber\\
&= (|v_t|- q_\mathrm{max})^2 \|X_t\|_2^2 + (1-\cos^2\theta_t) \|u_{t-1}\|_2^2.
\end{align}
By symmetry, we can assume without loss of generality that $v_t> q_\mathrm{max}$. In this case, since $|w_t|\leq \|w\|_\infty \leq kq_\mathrm{max}$,
\[
|v_t| - q_\mathrm{max} = v_t - q_\mathrm{max} = w_t - q_\mathrm{max} +\frac{\|u_{t-1}\|_2}{\|X_t\|_2}\cos\theta_t \leq (k-1)q_\mathrm{max} +\frac{\|u_{t-1}\|_2}{\|X_t\|_2}\cos\theta_t.
\]
Then \eqref{eq:bound-large-weights-eq2} becomes 
\begin{align*}
    \|u_t\|_2^2 &\leq \Bigl((k-1)q_\mathrm{max} +\frac{\|u_{t-1}\|_2}{\|X_t\|_2}\cos\theta_t \Bigr)^2 \|X_t\|_2^2 + (1-\cos^2\theta_t) \|u_{t-1}\|_2^2 \\
    &= (k-1)^2q_\mathrm{max}^2 \|X_t\|_2^2 + \|u_{t-1}\|_2^2 + 2(k-1) q_\mathrm{max} \langle X_t, u_{t-1}\rangle  \\
    &= \|(k-1)q_\mathrm{max} X_t + u_{t-1}\|_2^2 \\
    &\leq (\|u_{t-1}\|_2+(k-1)q_\mathrm{max}\|X_t\|_2)^2.
\end{align*}
This completes the proof. 
\end{proof}

We are now ready to bound the expected quantization error in the case when some weights have magnitude greater than $q_{\max}$.

\begin{theorem}\label{thm:large-weights}
Suppose that the  columns $X_t$ of  $X\in\mathbb{R}^{m\times N_0}$ are drawn independently from a probability distribution for which there exists $s\in(0,1)$ and $r>0$ such that $\|X_t\|_2 \leq r$ almost surely, and such that for all unit vector $u\in\mathbb{S}^{m-1}$ 
we have 
\begin{equation}\label{eq:large-weights-eq1}
\E\frac{\langle X_t, u\rangle^2}{\|X_t\|_2^2} \geq s^2.
\end{equation}
Let $\mathcal{A}$ be the alphabet in \eqref{eq-alphabet-midtread} with step size $\delta>0$, and the largest element $q_{\max}$. Let
$w\in\R^{N_0}$ be the weights associated with a neuron such that $\|w\|_\infty\leq kq_\mathrm{max}$ for some $k>1$. Let $n=|\{ t:|w_t|>q_\mathrm{max}\}|$ be the number of weights with magnitude greater than $q_\mathrm{max}$. Quantizing $w$ using \eqref{eq-update-rule-single}, we have 
\begin{equation}\label{eq:large-weights-eq0}
%\E \|Xw-Xq\|_2\leq nr(k-1)q_\mathrm{max} + \frac{\delta r}{s^2}.
\E \|Xw-Xq\|_2^2\leq \left( nr(k-1)q_\mathrm{max} + \frac{\delta r}{2s}\right)^2.
\end{equation}
\end{theorem}
\begin{proof}
Let $\theta_t$ be the angle between $X_t$ and $u_{t-1}$. It follows from \eqref{eq:large-weights-eq1} that 
\[
\E (\cos^2\theta_t \mid u_{t-1}) = \E\Bigl(\frac{\langle X_t, u_{t-1}\rangle^2}{\|X_t\|_2^2\|u_{t-1}\|_2^2} \; \Big| \;  u_{t-1} \Bigr) \geq s^2.
\]
Since $\|X_t\|_2\leq r$ almost surely and $\E (\cos^2\theta_t\mid u_{t-1}) \geq s^2$, by \cref{lemma:bound-large-weights}, we obtain
\begin{align}\label{eq:large-weights-eq2}
&\E (\|u_t\|_2^2 \mid u_{t-1}) \leq \\
&\begin{cases}
a\|u_{t-1}\|_2^2 + b &\text{if}\; \bigl| w_t+\frac{\|u_{t-1}\|_2}{\|X_t\|_2}\cos\theta_t \bigr|\leq q_{\max}, \\
\|u_{t-1}\|_2^2 &\text{if}\; \bigl| w_t+\frac{\|u_{t-1}\|_2}{\|X_t\|_2}\cos\theta_t \bigr|> q_{\max} \; \text{and} \; |w_t|\leq q_\mathrm{max}, \\
(\|u_{t-1}\|_2 + c)^2 &\text{if} \; \bigl| w_t+\frac{\|u_{t-1}\|_2}{\|X_t\|_2}\cos\theta_t \bigr|> q_{\max} \; \text{and} \; |w_t|> q_\mathrm{max}
\end{cases}  \nonumber 
\end{align}
where $a:=(1-s^2)$, $b:=\frac{\delta^2}{4}r^2$, and $c:=(k-1)rq_\mathrm{max}$. Define the indices $t_0:=0<t_1<\ldots<t_n < t_{n+1}:=N_0+1$ where $|w_{t_j}|>q_\mathrm{max}$ for $1\leq j\leq n$ and
%Since $n=|\{ t:|w_t|>q_\mathrm{max}\}|$, there exist $t_0:=0<t_1<\ldots<t_n < t_{n+1}:=N_0+1$ such that $|w_{t_j}|>q_\mathrm{max}$ for $1\leq j\leq n$. 
let 
\[
m_j:=\Bigl|\{t_{j-1}<t<t_j: \Bigl| w_t+\frac{\|u_{t-1}\|_2}{\|X_t\|_2}\cos\theta_t \Bigr|\leq q_{\max} \}\Bigr|, \quad 1\leq j \leq n+1.
\]
We first consider the case where $n=1$. %Then there exists $t_1$ such that $|w_{t_1}|>q_\mathrm{max}$. 
Applying the law of total expectation to the first two cases in \eqref{eq:large-weights-eq2}, one obtains
\begin{equation}\label{eq:large-weights-eq3}
\E\|u_{t_1-1}\|_2^2 \leq a^{m_1}\E \|u_0\|_2^2 + b(1+a+\ldots+a^{m_1-1})=b(1+a+\ldots+a^{m_1-1}).
\end{equation}
In the last equation, we used the fact $u_0=0$. Next, the last case in \eqref{eq:large-weights-eq2} can be used to bound $\E \|u_{t_1}\|_2^2$. Specifically, we have 
\begin{align}\label{eq:large-weights-eq4}
\E \|u_{t_1}\|_2^2 &= \E(\E (\|u_{t_1}\|_2^2 \mid u_{t_1-1}) \nonumber \\
&\leq \E \|u_{t_1-1}\|_2^2+2c\E\|u_{t_1-1}\|_2+c^2 &(\text{using} \; \eqref{eq:large-weights-eq2}) \nonumber \\
&\leq \E \|u_{t_1-1}\|_2^2+2c(\E\|u_{t_1-1}\|_2^2)^{\frac{1}{2}}+c^2  &(\text{by Jensen’s inequality}) \nonumber \\
&= ((\E \|u_{t_1-1}\|_2^2)^{\frac{1}{2}} + c)^2 \nonumber \\
&\leq \Bigl(c+ \sqrt{b(1+a+\ldots+a^{m_1-1})} \Bigr)^2 &(\text{using}\; \eqref{eq:large-weights-eq3}).
\end{align}
Since $|w_t|\leq q_\mathrm{max}$ for $t_1<t<t_2=N_0+1$, using \eqref{eq:large-weights-eq2}, we can derive 
\begin{align}\label{eq:large-weights-eq5}
\E\|u_{t_2 -1 }\|_2^2 &\leq a^{m_2}\E \|u_{t_1}\|_2^2 + b(1+a+\ldots+a^{m_2-1})\nonumber \\
&\leq a^{m_2}\left(c+ \sqrt{b\cdot\frac{1-a^{m_1}}{1-a}} \right)^2 + b\cdot\frac{1-a^{m_2}}{1-a}\nonumber  &(\text{using} \; \eqref{eq:large-weights-eq4}) \\ 
&= a^{m_2}c^2 + b\cdot\frac{1-a^{m_1+m_2}}{1-a} + 2a^{m_2}c \sqrt{\frac{b(1-a^{m_1})}{1-a}}  \nonumber \\
&\leq a^{m_2}c^2 + b\cdot\frac{1-a^{m_1+m_2}}{1-a} + 2a^{m_2/2}c \sqrt{\frac{b(1-a^{m_1+m_2})}{1-a}}  &(\text{since $0<a<1$})\nonumber \\
& \leq \left(c + \sqrt{\frac{b(1-a^{m_1+m_2})}{1-a}}   \right)^2.
\end{align}
Hence, we obtain $\E \|u_{N_0}\|_2^2 \leq \left(c + \sqrt\frac{{b}}{{1-a}}\right)^2$ when $n=1$. Proceeding by induction on $n$, we obtain
\begin{align}
\E\|u_{N_0}\|_2^2 &\leq \left(nc + \sqrt\frac{{b}}{{1-a}}\right)^2 = \left(nr(k-1)q_{\max} + \frac{\delta r}{2s}\right)^2.
\end{align}
Since $u_{N_0}=Xw-Xq$, we have $\E \|Xw-Xq\|_2^2\leq \left( nr(k-1)q_\mathrm{max} + \frac{\delta r}{2s}\right)^2$.
\end{proof}

Our numerical experiments in \Cref{sec:experiments} demonstrated that choosing our alphabet with $q_{\max}<\|w\|_\infty$ can yield better results than if we strictly conformed to choosing $\mathcal{A}$ with $q_{\max} \geq \|w\|_\infty$. Let us now see how \cref{thm:large-weights} can help explain these experimental results. 
First, recall from \eqref{eq-alphabet-midtread} that $q_\mathrm{max}=K\delta=2^{b-1}\delta$ where $b$ is the number of bits, and observe that the condition $\|w\|_\infty\leq kq_\mathrm{max}$ in \cref{thm:large-weights} implies that we can set $k=\|w\|_\infty/q_\mathrm{max}$. Thus \eqref{eq:large-weights-eq0}, coupled with Jensen's inequality, yields
\begin{equation}\label{eq:large-weights-eq}
\E \|Xw-Xq\|_2\leq nr(\|w\|_\infty-q_{\max}) +\frac{\delta r}{2s} = nr(\|w\|_\infty-2^{b-1}\delta) +\frac{\delta r}{2s}.
\end{equation}
Now, note that $s, r$ are fixed parameters that only depend on the input data distribution so for a fixed $b$, $n=n(\delta)=|\{ t:|w_t|>2^{b-1}\delta\}|$  is a decreasing function of  $\delta$. In other words, the right hand side of \eqref{eq:large-weights-eq} is the sum of an increasing function of $\delta$ and a decreasing function of $\delta$.
This means that there exists an optimal value of $\delta^*$ that minimizes the bound. In particular, it may not always be optimal to choose a large $\delta$ such that $\|w\|_\infty=2^{b-1}\delta$. This gives a theoretical justification for why the simple grid search we used in \Cref{sec:experiments} yielded better results.

\section{Theoretical Analysis for Gaussian Clusters}
In this section, we will prove \cref{thm-cluster}, which we first restate here for convenience. 

\noindent \cref{thm-cluster}: Let $X\in\mathbb{R}^{m\times N_0}$ be as in \eqref{cluster-eq1} and let $\mathcal{A}$ be as in \eqref{eq-alphabet-midtread}, with step size $\delta>0$ and the largest element $q_{\max}$. Let $p\in\mathbb{N}$, $K:=1+\sigma^{-2}\max_{1\leq i\leq d}\|z^{(i)}\|_\infty^2$, and
$w\in\R^{N_0}$ be the weights associated with a neuron, with $\|w\|_\infty\leq q_{\mathrm{max}}$. Quantizing $w$ using \eqref{eq-update-rule-single}, we have
 \[\prob\Bigl(\|Xw-Xq\|_2^2\geq 4pm^2K^2\delta^2\sigma^2 \log N_0\Bigr) \lesssim \frac{\sqrt{mK}}{N_0^p},\text{\quad and}\]
 \[\prob\Bigl(\max_{1\leq t\leq N_0}\|u_t\|_2^2\geq 4pm^2K^2\delta^2\sigma^2 \log N_0\Bigr) \lesssim \frac{\sqrt{mK}}{N_0^{p-1}}.\]
If the activation function $\varphi$ is $\xi$-Lipschitz continuous, then 
\[
\prob\biggl(\|\varphi(Xw)-\varphi(Xq)\|_2^2\geq 4pm^2K^2\xi^2\delta^2\sigma^2 \log N_0\biggr) 
\lesssim \frac{\sqrt{mK}}{N_0^p}.  
\]

\subsection{Proof of Theorem~\ref{thm-cluster}}\label{sec:thm-cluster-proof}
\noindent Due to $\|X_t\|_2^2=\sum_{i=1}^d\|Y_t^{(i)}\|_2^2$,  
\begin{equation}\label{cluster-eq2} 
\E\|X_t\|_2^2=\sum_{i=1}^d \E\|Y_t^{(i)}\|_2^2=\sum_{i=1}^d (n\sigma^2+n(z_t^{(i)})^2)=m\sigma^2+n\sum_{i=1}^d (z_t^{(i)})^2
\end{equation}
Additionally, given a unit vector $u=(u^{(1)}, u^{(2)}, \ldots, u^{(d)})\in\R^{m}$ with $u^{(i)}\in\R^n$, we have $\langle X_t,u \rangle=\sum_{i=1}^d\langle Y_t^{(i)}, u^{(i)}\rangle\sim\mathcal{N}\Bigl(\sum_{i=1}^d z_t^{(i)}u^{(i)\top} \mathbbm{1}_n, \sigma^2\Bigr)$.
In fact, once we get the lower bound of $\E\frac{\langle X_t, u\rangle^2}{\|X_t\|_2^2}$ as in \eqref{thm-main-result-single-assumption}, the quantization error for unbounded data \eqref{cluster-eq1} can be derived similarly to the proof of Theorem~\ref{thm-main-result-single}, albeit using different techniques. 
 It follows from the Cauchy-Schwarz inequality that
\begin{equation}\label{cluster-eq3}
\E\frac{\langle X_t,u\rangle^2}{\|X_t\|_2^2}\geq \frac{(\E|\langle X_t, u\rangle|)^2}{\E\|X_t\|_2^2}.
\end{equation}
$\E\|X_t\|_2^2$ is given by \eqref{cluster-eq2} while $\E|\langle X_t, u\rangle|$ can be evaluated by the following results.
\begin{lemma}\label{lemma-marginal-bound}
Let $Z\sim\mathcal{N}(\mu, \sigma^2)$ be a normally distributed random variable. Then 
\begin{equation}\label{lemma-marginal-bound-result}
\E|Z| \geq \sigma\sqrt{\frac{2}{\pi}}\biggl(1-\frac{4}{27\pi} \biggr).
\end{equation}
\end{lemma}
\begin{proof}
Let $\Psi(x)=\frac{1}{\sqrt{2\pi}}\int_{-\infty}^x e^{-t^2/2}\, dt$ be the normal cumulative distribution function. Due to $Z\sim\mathcal{N}(\mu, \sigma^2)$, the folded normal distribution $|Z|$ has mean $\E|Z|=\sigma\sqrt{\frac{2}{\pi}}e^{-\mu^2/2\sigma^2}+\mu(1-2\Psi(-\frac{\mu}{\sigma}))$. A well-known result \cite{foucart2013invitation, vershynin2018high} that can be used to bound $\Psi(x)$ is 
\begin{equation}\label{lemma-marginal-bound-eq1}
  \int_x^\infty e^{-t^2/2}\, dt \leq \min\biggl(\sqrt{\frac{\pi}{2}}, \frac{1}{x}\biggr)e^{-x^2/2},\quad \text{for}\, x>0.
\end{equation}
Additionally, in order to evaluate $\E|Z|$, it suffices to analyze the case $\mu\geq 0$ because one can replace $Z$ by $-Z$ without changing $|Z|$ when $\mu<0$. So we suppose $\mu\geq 0$. 

By \eqref{lemma-marginal-bound-eq1}, we obtain
\begin{align*}
\E|Z|&=\sigma\sqrt{\frac{2}{\pi}}e^{-\mu^2/2\sigma^2}+\mu-2\mu\Psi(-\mu/\sigma) 
=\sigma\sqrt{\frac{2}{\pi}}e^{-\mu^2/2\sigma^2}+\mu-\mu\sqrt{\frac{2}{\pi}}\int_{\mu/\sigma}^\infty e^{-t^2/2}\, dt\\
&\geq \sigma\sqrt{\frac{2}{\pi}}e^{-\mu^2/2\sigma^2}+\mu-\min \biggl( \mu ,\sigma\sqrt{\frac{2}{\pi}}\biggr)e^{-\mu^2/2\sigma^2}.
\end{align*}
If $\mu\geq\sigma\sqrt{\frac{2}{\pi}}$, then one can easily get 
$\E|Z|\geq \mu\geq \sigma\sqrt{\frac{2}{\pi}}$.
Further, if $0\leq\mu<\sigma\sqrt{\frac{2}{\pi}}$, then $\E|Z|\geq (\sigma\sqrt{2/\pi}-\mu)e^{-\mu^2/2\sigma^2}+\mu$. Due to $e^x\geq 1+x$ for all $x\in\R$, one can get 
\[
\E|Z|\geq (\sigma\sqrt{2/\pi}-\mu)(1-\mu^2/2\sigma^2)+\mu =\frac{1}{2\sigma^2}\mu^3-\frac{1}{\sigma\sqrt{2\pi}}\mu^2+\sigma\sqrt{\frac{2}{\pi}}\geq \sigma\sqrt{\frac{2}{\pi}}\biggl(1-\frac{4}{27\pi}\biggr).
\]
In the last inequality, we optimized in $\mu\in(0,\sigma\sqrt{2/\pi})$ and thus chose $\mu=\frac{2}{3}\cdot \sigma\sqrt{\frac{2}{\pi}}$. 
\end{proof}
% \noindent (II) Now we assume $\mu<0$. It follows from \eqref{lemma-marginal-bound-eq1} that
% \begin{align*}
% \E|Z|&=\sigma\sqrt{\frac{2}{\pi}}e^{-\mu^2/2\sigma^2}+\mu-2\mu\Psi(-\mu/\sigma) \\
% &=\sigma\sqrt{\frac{2}{\pi}}e^{-\mu^2/2\sigma^2}+\mu-2\mu\biggl(1-\frac{1}{\sqrt{2\pi}}\int_{-\mu/\sigma}^\infty e^{-t^2/2}\, dt\biggr)\\
% &\geq \sigma\sqrt{\frac{2}{\pi}}e^{-\mu^2/2\sigma^2}-\mu+\max \biggl( \mu ,-\sigma\sqrt{\frac{2}{\pi}}\biggr) e^{-\mu^2/2\sigma^2}.
% \end{align*}
% If $\mu\leq - \sigma\sqrt{\frac{2}{\pi}}$, then we have $\E|Z|\geq -\mu\geq \sigma\sqrt{\frac{2}{\pi}}$. Moreover, when $-\sigma\sqrt{\frac{2}{\pi}}<\mu<0$, using the fact $e^x\geq 1+x$ again, we obtain 
% \[
% \E|Z|\geq (\sigma\sqrt{2/\pi}+\mu)(1-\mu^2/2\sigma^2)-\mu =-\frac{1}{2\sigma^2}\mu^3-\frac{1}{\sigma\sqrt{2\pi}}\mu^2+\sigma\sqrt{\frac{2}{\pi}}\geq \sigma\sqrt{\frac{2}{\pi}}\biggl(1-\frac{4}{27\pi}\biggr). 
% \]
% The last inequality was derived by optimizing in $\mu\in(-\sigma\sqrt{2/\pi},0)$ and choosing $\mu=-\frac{2}{3}\cdot \sigma \sqrt{\frac{2}{\pi}}$. Combining above results, \eqref{lemma-marginal-bound-result} holds.

\begin{lemma}\label{thm-cluster-bound}
Let clustered data $X=[X_1,X_2,\ldots,X_{N_0}]\in\R^{m\times N_0}$ be defined as in \eqref{cluster-eq1} and $u\in\R^{m}$ be a unit vector. Then, for $1\leq t\leq N_0$, we have 
\begin{equation}\label{cluster-lower-bound-eq}
\E\frac{\langle X_t, u \rangle^2}{\|X_t\|_2^2}\geq \frac{5}{9}\cdot\frac{ \sigma^2}{m (\sigma^2 +  \max_{1\leq i \leq d}\|z^{(i)}\|_\infty^2)}.
\end{equation} 
\end{lemma}
\begin{proof}
Since $\langle X_t, u\rangle$ is normally distributed with variance $\sigma^2$, \eqref{lemma-marginal-bound-result} implies 
$\E|\langle X_t, u\rangle| \geq \sigma \sqrt{\frac{2}{\pi}}\biggl(1-\frac{4}{27\pi}\biggr)$.
Plugging the inequality above and \eqref{cluster-eq2} into \eqref{cluster-eq3}, we obtain
\[
\E\frac{\langle X_t, u \rangle^2}{\|X_t\|_2^2}\geq \frac{(\E|\langle X_t, u\rangle|)^2}{\E\|X_t\|_2^2}\geq \frac{2(1-\frac{4}{27\pi})^2}{\pi}\cdot\frac{ \sigma^2}{m\sigma^2 + n \sum_{i=1}^d (z_t^{(i)})^2} \geq \frac{5}{9}\cdot \frac{ \sigma^2}{m\sigma^2 + n \sum_{i=1}^d (z_t^{(i)})^2}.
\]
Therefore, \eqref{cluster-lower-bound-eq} holds due to $(z_t^{(i)})^2\leq \|z^{(i)}\|_\infty^2\leq \max_{1\leq i\leq d} \|z^{(i)}\|_\infty^2$ and $m=nd$.
\end{proof}

Now we are ready to prove Theorem~\ref{thm-cluster}. 
\begin{proof}
Let $\alpha>0$ and $\eta>0$. By using exactly the same argument as in \eqref{thm-main-result-single-eq1}, at the $t$-th step of \eqref{eq-update-rule-single}, we have 
\begin{equation}\label{thm-cluster-eq1}
\prob(\|u_t\|_2^2\geq \alpha)\leq e^{-\eta\alpha}\E e^{\eta\|u_t\|_2^2}.
\end{equation}
Moreover, Lemma~\ref{lemma-bound} implies
\begin{equation}\label{thm-cluster-eq2}
\E e^{\eta\|u_{t}\|_2^2} \leq \max\Bigl\{
\E(e^{\frac{\eta\delta^2}{4}\|X_t\|_2^2}e^{\eta\|u_{t-1}\|_2^2(1-\cos^2\theta_t)}),
\E e^{\eta\|u_{t-1}\|_2^2} \Bigr\}
\end{equation} 
Until now our analysis here has been quite similar to what we did for bounded input data in Theorem~\ref{thm-main-result-single}. Nevertheless, unlike Theorem~\ref{thm-main-result-single}, we will control the moment generating function of $\|X_t\|_2^2$ because $\|X_t\|_2^2$ is unbounded. Specifically, applying the Cauchy-Schwarz inequality and Lemma~\ref{lemma-bound} (2) with $\beta=2$, we obtain
\begin{align}\label{thm-cluster-eq3}
\E(e^{\frac{\eta\delta^2}{4}\|X_t\|_2^2}& e^{\eta\|u_{t-1}\|_2^2(1-\cos^2\theta_t)}\mid \mathcal{F}_{t-1}) \leq \bigl(\E e^{\frac{\eta\delta^2}{2}\|X_t\|_2^2}\bigr)^{\frac{1}{2}}\bigl(\E( e^{2\eta\|u_{t-1}\|_2^2(1-\cos^2\theta_t)}\mid \mathcal{F}_{t-1})\bigr)^{\frac{1}{2}} \nonumber \\
&\leq \bigl( \E e^{\frac{\eta\delta^2}{2}\|X_t\|_2^2}\bigr)^{\frac{1}{2}}\bigl( -\E(\cos^2\theta_t\mid\F_{t-1})(e^{2\eta\|u_{t-1}\|_2^2}-1)+e^{2\eta\|u_{t-1}\|_2^2}\bigr)^{\frac{1}{2}}
\end{align}
In the first step, we also used the fact that $X_t$ is independent of $\mathcal{F}_{t-1}$. By \eqref{cluster-lower-bound-eq}, we have
\[
\E(\cos^2\theta_t\mid\F_{t-1}) = \E\biggl(\frac{\langle X_t, u_{t-1}\rangle^2}{\|X_t\|_2^2\|u_{t-1}\|_2^2}\, \Big|\, \F_{t-1}\biggr)
\geq \frac{5}{9mK} =: s^2.
\]
Plugging the inequality above into \eqref{thm-cluster-eq3}, we get
\begin{align}\label{thm-cluster-eq4}
\E(e^{\frac{\eta\delta^2}{4}\|X_t\|_2^2}e^{\eta\|u_{t-1}\|_2^2(1-\cos^2\theta_t)}\mid \mathcal{F}_{t-1}) &\leq \bigl(\E e^{\frac{\eta\delta^2}{2}\|X_t\|_2^2}\bigr)^{\frac{1}{2}}\Bigl(-s^2(e^{2\eta\|u_{t-1}\|_2^2}-1)+e^{2\eta\|u_{t-1}\|_2^2}\Bigr)^{\frac{1}{2}} \nonumber\\
&= \bigl(\E e^{\frac{\eta\delta^2}{2}\|X_t\|_2^2}\bigr)^{\frac{1}{2}} \Bigl(e^{2\eta\|u_{t-1}\|_2^2}(1-s^2)+s^2 \Bigr)^{\frac{1}{2}} \nonumber \\
&\leq \bigl(\E e^{\frac{\eta\delta^2}{2}\|X_t\|_2^2}\bigr)^{\frac{1}{2}}(e^{\eta\|u_{t-1}\|_2^2}(1-s^2)^{\frac{1}{2}}+s) \nonumber  \\
&\leq \bigl(\E e^{\frac{\eta\delta^2}{2}\|X_t\|_2^2}\bigr)^{\frac{1}{2}}(e^{\eta\|u_{t-1}\|_2^2}(1-\frac{1}{2}s^2)+s)
\end{align}
where the last two inequalities hold due to $(x^2+y^2)^{\frac{1}{2}}\leq |x|+|y|$ for all $x,y\in\R$, and $(1-x)^{\frac{1}{2}}\leq 1-\frac{1}{2}x$ whenever $x\leq 1$. 

Now we evaluate $\E e^{\frac{\eta\delta^2}{2}\|X_t\|_2^2}$ and note that
\begin{equation}\label{thm-cluster-eq5}
\E e^{\frac{\eta\delta^2}{2}\|X_t\|_2^2} = \E\exp\Bigl( \frac{\eta\delta^2}{2}\sum_{i=1}^d \|Y_t^{(i)}\|_2^2\Bigr) 
= \prod_{i=1}^d \E \exp\Bigl(\frac{\eta\delta^2}{2}\|Y_t^{(i)}\|_2^2\Bigr).   
\end{equation} 
Since $Y_t^{(i)}\sim\mathcal{N}(z_t^{(i)}\mathbbm{1}_n, \sigma^2 I_n)$, we have 
\begin{align*}
\E \exp\Bigl(\frac{\eta\delta^2}{2} & \|Y_t^{(i)}\|_2^2 \Bigr) = \biggl[\frac{1}{\sigma\sqrt{2\pi}}\int_\R \exp\biggl(-\frac{(x-z_t^{(i)})^2}{2\sigma^2}+\frac{\eta\delta^2x^2}{2}\biggr)\, dx \biggr]^n \\
&=\biggl\{\frac{1}{\sigma\sqrt{2\pi}}\cdot \exp\biggl(\frac{\eta\delta^2 (z_t^{(i)})^2}{2-2\eta\delta^2\sigma^2}\biggr)\int_\R \exp\biggl[ -\frac{1-\eta\delta^2\sigma^2}{2\sigma^2}\biggl(x-\frac{z_t^{(i)}}{1-\eta\delta^2\sigma^2} \biggr)^2\biggr]\, dx \bigg\}^n \\
&= \biggl[ (1-\eta\delta^2\sigma^2)^{-\frac{1}{2}} \exp\biggl(\frac{\eta\delta^2 (z_t^{(i)})^2}{2-2\eta\delta^2\sigma^2}\biggr) \biggr]^n 
\end{align*}
where the last equality holds if $\eta\delta^2\sigma^2<1$ and we use the integral of the normal density function:
\[
\biggl(\frac{1-\eta\delta^2\sigma^2}{2\pi\sigma^2}\biggr)^{\frac{1}{2}}\int_\R \exp\biggl[ -\frac{1-\eta\delta^2\sigma^2}{2\sigma^2}\biggl(x-\frac{z_t^{(i)}}{1-\eta\delta^2\sigma^2} \biggr)^2\biggr]\, dx = 1.
\]
Notice that 
$\frac{1}{1-x}\leq 1+2x$ for $x\in[0,\frac{1}{2}]$ and $1+x\leq e^x$ for all $x\in\R$. Now, we suppose $\eta\delta^2\sigma^2\leq \frac{1}{2}$ and thus
$(1-\eta\delta^2\sigma^2)^{-\frac{1}{2}} =\Bigl(\frac{1}{1-\eta\delta^2\sigma^2} \Bigr)^{\frac{1}{2}} \leq (1+2\eta\delta^2\sigma^2)^{\frac{1}{2}}\leq e^{\eta\delta^2\sigma^2}$.
It follows that 
\begin{align}\label{thm-cluster-eq6} 
\E\exp\Bigl(\frac{\eta\delta^2}{2} \|Y_t^{(i)}\|_2^2 \Bigr) &\leq \biggl[\exp\biggl(\eta\delta^2\sigma^2+\frac{\eta\delta^2(z_t^{(i)})^2}{2-2\eta\delta^2\sigma^2} \biggr) \biggr]^n \leq \biggl[\exp\biggl(\eta\delta^2\sigma^2+\eta\delta^2(z_t^{(i)})^2\biggr)\biggr]^n \nonumber \\
&\leq\exp\biggl( n\eta\delta^2\sigma^2\biggl(1+\frac{\|z^{(i)}\|_\infty^2}{\sigma^2}\biggr) \biggr) \nonumber \\
&\leq \exp(nK\eta\delta^2\sigma^2)
\end{align} 
Substituting \eqref{thm-cluster-eq6} into \eqref{thm-cluster-eq5}, we get
\begin{equation}\label{thm-cluster-eq7}
\E e^{\frac{\eta\delta^2}{2}\|X_t\|_2^2} \leq e^{ndK\eta\delta^2\sigma^2}=e^{mK\eta\delta^2\sigma^2}.
\end{equation} 
Combining \eqref{thm-cluster-eq4} and \eqref{thm-cluster-eq7}, if $\eta\delta^2\sigma^2\leq\frac{1}{2}$, then 
\begin{align}\label{thm-cluster-eq8}
\E(e^{\frac{\eta\delta^2}{4}\|X_t\|_2^2}e^{\eta\|u_{t-1}\|_2^2(1-\cos^2\theta_t)})  &=\E\Bigl(\E(e^{\frac{\eta\delta^2}{4}\|X_t\|_2^2}e^{\eta\|u_{t-1}\|_2^2(1-\cos^2\theta_t)}\mid \mathcal{F}_{t-1}) \Bigr)\nonumber \\ 
&\leq \E\Bigl( e^{\frac{1}{2}mK\eta\delta^2\sigma^2}(e^{\eta\|u_{t-1}\|_2^2}(1-\frac{1}{2}s^2)+s)  \Bigr) \nonumber \\
&= e^{\frac{1}{2}mK\eta\delta^2\sigma^2}(1-\frac{1}{2}s^2)\E e^{\eta\|u_{t-1}\|_2^2}+ s e^{\frac{1}{2}mK\eta\delta^2\sigma^2}\nonumber  \\
&=: a\E e^{\eta\|u_{t-1}\|_2^2} +b 
\end{align}
with $a:=(1-s^2/2)e^{\frac{1}{2}mK\eta\delta^2\sigma^2}$ and $b:=s e^{\frac{1}{2}mK\eta\delta^2\sigma^2}$. Plugging \eqref{thm-cluster-eq8} into \eqref{thm-cluster-eq2}, we have $\E e^{\eta\|u_t\|_2^2}\leq \max\{a\E e^{\eta\|u_{t-1}\|_2^2}+b, \E e^{\eta\|u_{t-1}\|_2^2} \}$. Next, similar to the argument in \eqref{thm-main-result-single-eq4}, iterating expectations yields 
$\E e^{\eta\|u_t\|_2^2} \leq a^{t_0}\E(e^{\eta \|u_0\|_2^2})+b(1+a+\ldots+a^{t_0}) 
= a^{t_0} + \frac{b(1-a^{t_0})}{1-a} \leq 1 + \frac{b}{1-a}$
where the last inequality holds if $a:=(1-s^2/2)e^{mK\eta \delta^2\sigma^2/2}<1$. So we can now choose $\eta=\frac{-\log(1-s^2/2)}{mK\delta^2\sigma^2}$, which satisfies  $\eta\delta^2\sigma^2\in[0,1/2]$ as required from before. Indeed, due to $m, K\geq 1$ and $s^2=\frac{5}{9Km}\leq \frac{5}{9}$, we have $\eta\delta^2\sigma^2=\frac{-\log(1-s^2/2)}{mK}\leq -\log(1-\frac{5}{18})<\frac{1}{2}$. Then we get $a=(1-\frac{1}{2}s^2)^{1/2}$ and $b=s(1-\frac{1}{2}s^2)^{-1/2}$ . It follows from  \eqref{thm-cluster-eq1} and $s^2=\frac{5}{9mK}$ that
\begin{align*}
    \prob(\|u_t\|_2^2\geq \alpha)&\leq e^{-\eta\alpha}\biggl(1+\frac{b}{1-a}\biggr) =\exp\biggl(\frac{\alpha\log(1-s^2/2)}{mK\delta^2\sigma^2}\biggr)\biggl(1+\frac{s(1-\frac{1}{2}s^2)^{-1/2}}{1-\sqrt{1-s^2/2}}\biggr) \\ 
    &\leq \exp\biggl(\frac{-\alpha s^2}{2mK\delta^2\sigma^2}\biggr)\biggl( 1 + \frac{s(1-\frac{1}{2}s^2)^{-1/2}+s}{s^2/2} \biggr) \qquad (\text{since}\, \log(1+x)\leq x)
\\    &= \exp\biggl(\frac{-\alpha s^2}{2mK\delta^2\sigma^2}\biggr)\biggl( 1 +  2\frac{(1-\frac{1}{2}s^2)^{-1/2}+1}{s} \biggr)
\\ & = \exp\biggl(-\frac{5\alpha} {18 m^2K^2\delta^2\sigma^2}\biggr) \biggl[ 1+ 6\sqrt{\frac{mK}{5}}\biggl(1-\frac{5}{18mK}\biggr)^{-1/2}+6\sqrt{\frac{mK}{5}} \biggr]\\
&\leq c\sqrt{mK} \exp\biggl(-\frac{\alpha}{4m^2K^2\delta^2\sigma^2}\biggr)
\end{align*} 
where $c>0$ is an absolute constant. Pick $\alpha=4m^2K^2\delta^2\sigma^2 \log(N_0^p) $ to get
\begin{equation}\label{thm-cluster-eq10}
\prob\biggl(\|u_t\|_2^2\geq 4pm^2K^2\delta^2\sigma^2 \log N_0\biggr) \leq c\sqrt{mK}N_0^{-p}.
\end{equation}
From \eqref{thm-cluster-eq10} we can first conclude, by setting $t=N_0$ and using the fact $u_{N_0}=Xw-Xq$, that 
\[ 
\prob\biggl(\|Xw-Xq\|_2^2\geq 4pm^2K^2\delta^2\sigma^2 \log N_0\biggr) \leq \frac{c\sqrt{mK}}{N_0^p}.
\] 
If the activation function $\varphi$ is $\xi$-Lipschitz, then $\|\varphi(Xw)-\varphi(Xq)\|_2\leq \xi\|Xw-Xq\|_2$ and thus
\[
\prob\biggl(\|\varphi(Xw)-\varphi(Xq)\|_2^2\geq 4pm^2K^2\xi^2\delta^2\sigma^2 \log N_0\biggr) \leq \frac{c\sqrt{mK}}{N_0^p}.
\]
Moreover, applying a union bound over $t$, yields 
\[
\prob\biggl(\max_{1\leq t\leq N_0}\|u_t\|_2^2\geq 4pm^2K^2\delta^2\sigma^2 \log N_0\biggr) \leq \frac{c\sqrt{mK}}{N_0^{p-1}}.
\] 
\end{proof}

\section{Theoretical Analysis for Sparse GPFQ}\label{sec:theory-sparse-quantization}
In this section, we will show that \cref{thm-main-result-single-sparse} and \cref{thm-cluster-sparse} (restated here for convenience) hold.\medskip 

\noindent \cref{thm-main-result-single-sparse}: Under the conditions of \cref{thm-main-result-single}, we have the following. \\
$(a)$ Quantizing $w$ using \eqref{eq-update-rule-single-soft} with the alphabet $\A$ in \eqref{eq-alphabet-midtread}, we have 
\[
\prob\Bigl(\|Xw-Xq\|_2^2\leq \frac{r^2(2\lambda+\delta)^2}{s^2}\log N_0\Bigr) \geq 1- \frac{1}{N_0^2}\Bigl(2+\frac{1}{\sqrt{1-s^2}}\Bigr).
\]
% \item $\prob\Bigl(\max_{1\leq t\leq N_0}\|u_t\|_2^2\leq \frac{r^2(2\lambda+\delta)^2}{s^2}\log N_0\Bigr) \geq 1- \frac{1}{N_0}\Bigl(2+\frac{1}{\sqrt{1-s^2}}\Bigr)$.
$(b)$ Quantizing $w$ using \eqref{eq-update-rule-single-hard} with the alphabet $\widetilde{\A}$ in \eqref{eq-alphabet-midtread-variant}, we have
\[ 
\prob\Bigl(\|Xw-Xq\|_2^2\leq \frac{r^2\max\{2\lambda,\delta\}^2}{s^2}\log N_0\Bigr) \geq 1- \frac{1}{N_0^2}\Bigl(2+\frac{1}{\sqrt{1-s^2}}\Bigr).
\]

\noindent \cref{thm-cluster-sparse}: Under the assumptions of \cref{thm-cluster}, the followings inequalities hold.\\
$(a)$ Quantizing $w$ using \eqref{eq-update-rule-single-soft} with the alphabet $\A$ in \eqref{eq-alphabet-midtread}, we have
\[
\prob\Bigl(\|Xw-Xq\|_2^2\geq 4pm^2K^2(2\lambda+\delta)^2\sigma^2 \log N_0\Bigr) \lesssim \frac{\sqrt{mK}}{N_0^p}.
\]
$(b)$ Quantizing $w$ using \eqref{eq-update-rule-single-hard} with the alphabet $\widetilde{\A}$ in \eqref{eq-alphabet-midtread-variant}, we have
\[
\prob\Bigl(\|Xw-Xq\|_2^2\geq 4pm^2K^2\max\{2\lambda,\delta\}^2\sigma^2 \log N_0\Bigr) \lesssim \frac{\sqrt{mK}}{N_0^p}.
\]

Note that the difference between the sparse GPFQ and the GPFQ in \eqref{eq-update-rule-single} is the usage of thresholding functions. So the key point is to adapt \cref{lemma-diff} and \cref{lemma-bound} for those changes. 

\subsection{Sparse GPFQ with Soft Thresholding}
We first focus on the error analysis for \eqref{eq-update-rule-single-soft} which needs the following lemmata. 
\begin{lemma}\label{lemma-diff-soft}
Let $\mathcal{A}$ be one of the alphabets defined in \eqref{eq-alphabet-midtread} with step size $\delta>0$, and the largest element $q_{\text{max}}$. Let $\theta_t:=\angle(X_t, u_{t-1})$ be the angle between $X_t$ and $u_{t-1}$. Suppose that $w\in\R^{N_0}$ satisfies $\|w\|_\infty\leq q_{\mathrm{max}}$, and consider the quantization scheme given by \eqref{eq-update-rule-single-soft}. Then, for $t=1,2,\ldots,N_0$, we have 
\begin{equation}\label{lemma-diff-soft-eq}
\|u_t\|_2^2-\|u_{t-1}\|_2^2\leq 
\begin{cases}
\frac{(2\lambda+\delta)^2}{4}\|X_t\|_2^2-\|u_{t-1}\|_2^2\cos^2\theta_t &\text{if}\; \Bigl| w_t+\frac{\|u_{t-1}\|_2}{\|X_t\|_2}\cos\theta_t \Bigr|\leq q_{\text{max}}+\lambda, \\
0 &\text{otherwise}.
\end{cases}  
\end{equation}
\end{lemma}

\begin{proof}
By applying exactly the same argument as in \cref{lemma-diff}, one can get 
\begin{equation}\label{lemma-diff-soft-eq1}
\|P_{X_t}(u_t)\|_2^2= \Bigl( w_t+\frac{\|u_{t-1}\|_2}{\|X_t\|_2}\cos\theta_t -q_t\Bigr)^2\|X_t\|_2^2.
\end{equation}
and
\[
\Bigl( w_t+\frac{\|u_{t-1}\|_2}{\|X_t\|_2}\cos\theta_t -q_t\Bigr)^2-\Bigl(\frac{\|u_{t-1}\|_2}{\|X_t\|_2}\cos\theta_t \Bigr)^2 =
\Bigl(\underbrace{ w_t+\frac{2\|u_{t-1}\|_2}{\|X_t\|_2}\cos\theta_t -q_t}_{\text{(I)}}\Bigr)(\underbrace{w_t-q_t}_{\text{(II)}}),
\]
where $|w_t|\leq q_{\text{max}}$ and $q_t=\mathcal{Q}\circ  s_\lambda\Bigl(w_t+\frac{\|u_{t-1}\|_2}{\|X_t\|_2}\cos\theta_t \Bigr)$. We proceed by going through the cases.

First, if $\Bigl( w_t+\frac{\|u_{t-1}\|_2}{\|X_t\|_2}\cos\theta_t \Bigr)> q_{\text{max}}+\lambda$, then $q_t=q_{\text{max}}$ and thus $\lambda \leq q_t-w_t+\lambda < \frac{\|u_{t-1}\|_2}{\|X_t\|_2}\cos\theta_t$. So 
$\text{(I)}> w_t+2(q_t-w_t+\lambda)-q_t=q_t-w_t+2\lambda \geq 2\lambda$
and $\text{(II)}\leq 0$. Moreover, if $\Bigl( w_t+\frac{\|u_{t-1}\|_2}{\|X_t\|_2}\cos\theta_t \Bigr)<- q_{\text{max}}-\lambda$, then $q_t=-q_{\text{max}}$ and $\frac{\|u_{t-1}\|_2}{\|X_t\|_2}\cos\theta_t < q_t-w_t-\lambda \leq -\lambda$. Hence, $\text{(I)} < w_t+2(q_t-w_t-\lambda)-q_t=q_t-w_t-2\lambda\leq -2\lambda $ and $\text{(II)}\geq 0$. It follows that 
\begin{equation}\label{lemma-diff-soft-eq3}
\Bigl( w_t+\frac{\|u_{t-1}\|_2}{\|X_t\|_2}\cos\theta_t  -q_t\Bigr)^2\leq\Bigl(\frac{\|u_{t-1}\|_2}{\|X_t\|_2}\cos\theta_t \Bigr)^2 
\end{equation}
when $\Bigl| w_t+\frac{\|u_{t-1}\|_2}{\|X_t\|_2}\cos\theta_t \Bigr|> q_{\text{max}}+\lambda$.

Now, assume that $\Bigl| w_t+\frac{\|u_{t-1}\|_2}{\|X_t\|_2}\cos\theta_t \Bigr|\leq q_{\text{max}}+\lambda$. In this case, let  $v_t:=s_\lambda\Bigl(w_t+\frac{\|u_{t-1}\|_2}{\|X_t\|_2}\cos\theta_t \Bigr)$. Then $|v_t|\leq q_\mathrm{max}$ and $\Bigl|w_t+\frac{\|u_{t-1}\|_2}{\|X_t\|_2}\cos\theta_t -v_t \Bigr|\leq \lambda$.
Since $q_t=\mathcal{Q}(v_t)$, we obtain 
\begin{align}\label{lemma-diff-soft-eq4} 
\Bigl( w_t+\frac{\|u_{t-1}\|_2}{\|X_t\|_2}\cos\theta_t -q_t\Bigr)^2 &= \Bigl| w_t+\frac{\|u_{t-1}\|_2}{\|X_t\|_2}\cos\theta_t -v_t + v_t - q_t\Bigr|^2  \nonumber\\
&\leq \biggl( \Bigl| w_t+\frac{\|u_{t-1}\|_2}{\|X_t\|_2}\cos\theta_t -v_t\Bigr| + |v_t - q_t|\biggl)^2 \nonumber \\
&\leq \Bigl(\lambda +\frac{\delta}{2} \Bigr)^2.
\end{align}
Applying \eqref{lemma-diff-soft-eq3} and \eqref{lemma-diff-soft-eq4} to \eqref{lemma-diff-soft-eq1}, one can get 
\begin{equation}\label{lemma-diff-soft-eq5}
    \|P_{X_t}(u_t)\|_2^2\leq 
\begin{cases}
\frac{(2\lambda+\delta)^2}{4}\|X_t\|_2^2 &\text{if}\; \Bigl| w_t+\frac{\|u_{t-1}\|_2}{\|X_t\|_2}\cos\theta_t \Bigr|\leq q_{\text{max}}+\lambda , \\
 \|u_{t-1}\|_2^2\cos^2\theta_t  &\text{otherwise}.
\end{cases}
\end{equation}
Again, by the same discussion after \eqref{lemma-diff-eq2} in \cref{lemma-diff}, we have 
$\|u_t\|_2^2 -\|u_{t-1}\|_2^2= \|P_{X_t}(u_t)\|_2^2-\|u_{t-1}\|_2^2\cos^2\theta_t$.
Replacing $\|P_{X_t}(u_t)\|_2^2$ with its upper bounds in \eqref{lemma-diff-soft-eq5}, we obtain \eqref{lemma-diff-soft-eq}.
\end{proof}

\begin{lemma}\label{lemma-bound-soft}
Let $\mathcal{A}$ be one of the alphabets defined in \eqref{eq-alphabet-midtread} with step size $\delta>0$, and the largest element $q_{\max}$. Suppose that $w\in\R^{N_0}$  satisfies $\|w\|_\infty\leq q_{\mathrm{max}}$, and consider the quantization scheme  given by \eqref{eq-update-rule-single-soft}. Additionally, denote the information of the first $t-1$ quantization steps by a $\sigma$-algebra $\mathcal{F}_{t-1}$, and let $\beta, \eta>0$, $s^2\in (0,1)$. Then the following results hold for $t=1,2,\ldots,N_0$.
\begin{enumerate} 
    \item 
    $\E e^{\eta\|u_{t}\|_2^2} \leq \max\Bigl\{
    \E(e^{\frac{\eta(2\lambda+\delta)^2}{4}\|X_t\|_2^2}e^{\eta\|u_{t-1}\|_2^2(1-\cos^2\theta_t)}),
    \E e^{\eta\|u_{t-1}\|_2^2} \Bigr\} $. 
    \item
    $\E(e^{\eta\beta\|u_{t-1}\|_2^2(1-\cos^2\theta_t)}\mid \mathcal{F}_{t-1}) \leq 
    -\E(\cos^2\theta_t\mid\mathcal{F}_{t-1})(e^{\eta\beta\|u_{t-1}\|_2^2}-1)+e^{\eta\beta\|u_{t-1}\|_2^2}$.
\end{enumerate}
Here, $\theta_t$ is the angle between $X_t$ and $u_{t-1}$.
\end{lemma}
\begin{proof}
Similar to \cref{lemma-bound}, the inequality (1) follows immediately from \eqref{lemma-diff-soft-eq}. The proof of part (2) is identical with the one in \cref{lemma-bound}.
\end{proof}

Now we are ready to prove \cref{thm-main-result-single-sparse} as follows.
\begin{proof}
The only difference between \cref{lemma-bound} and its analogue \cref{lemma-bound-soft} is that $\delta^2$ in \cref{lemma-bound} is replaced by $(2\lambda+\delta)^2$. Note that \cref{lemma-bound} was used in the proof of both \cref{thm-main-result-single} and \cref{thm-cluster} in which $\delta^2$ serves as a coefficient. Hence, by substituting $\delta^2$ with $(2\lambda+\delta)^2$, every step in the proof still works and thus \cref{thm-main-result-single-sparse} holds. 
\end{proof}

\subsection{Sparse GPFQ with Hard Thresholding}
Now we navigate to the error analysis for \eqref{eq-update-rule-single-hard}. Again, \cref{lemma-diff} and \cref{lemma-bound} are altered as follows. 
\begin{lemma}\label{lemma-diff-hard}
Let $\widetilde{\A}$ be one of the alphabets defined in \eqref{eq-alphabet-midtread-variant} with step size $\delta>0$, the largest element $q_{\text{max}}$, and threshold $\lambda\in(0,q_\mathrm{max})$. Let $\theta_t:=\angle(X_t, u_{t-1})$ be the angle between $X_t$ and $u_{t-1}$. Suppose that $w\in\R^{N_0}$ satisfies $\|w\|_\infty\leq q_{\mathrm{max}}$, and consider the quantization scheme given by \eqref{eq-update-rule-single-hard}. Then, for $t=1,2,\ldots,N_0$, we have 
\begin{equation}\label{lemma-diff-hard-eq}
\|u_t\|_2^2-\|u_{t-1}\|_2^2\leq 
\begin{cases}
\frac{\max\{2\lambda,\delta\}^2}{4}\|X_t\|_2^2-\|u_{t-1}\|_2^2\cos^2\theta_t &\text{if}\; \Bigl| w_t+\frac{\|u_{t-1}\|_2}{\|X_t\|_2}\cos\theta_t \Bigr|\leq q_{\text{max}}, \\
0 &\text{otherwise}.
\end{cases}  
\end{equation}
\end{lemma}

\begin{proof}
By applying exactly the same argument as in \cref{lemma-diff}, we obtain  
\begin{equation}\label{lemma-diff-hard-eq1}
\|P_{X_t}(u_t)\|_2^2= \Bigl( w_t+\frac{\|u_{t-1}\|_2}{\|X_t\|_2}\cos\theta_t -q_t\Bigr)^2\|X_t\|_2^2.
\end{equation}
where $|w_t|\leq q_{\text{max}}$ and $q_t=\mathcal{Q}\circ  h_\lambda\Bigl(w_t+\frac{\|u_{t-1}\|_2}{\|X_t\|_2}\cos\theta_t \Bigr)$. Due to $\lambda\in(0, q_\mathrm{max})$, we have  $\mathcal{Q}\circ h_\lambda(z) = \mathcal{Q}(z)$ for $|z|>q_\mathrm{max}$. Thus, it follows from the discussion in \cref{lemma-diff} that 
\begin{equation}\label{lemma-diff-hard-eq3}
\Bigl( w_t+\frac{\|u_{t-1}\|_2}{\|X_t\|_2}\cos\theta_t  -q_t\Bigr)^2\leq\Bigl(\frac{\|u_{t-1}\|_2}{\|X_t\|_2}\cos\theta_t \Bigr)^2 
\end{equation}
when $\Bigl| w_t+\frac{\|u_{t-1}\|_2}{\|X_t\|_2}\cos\theta_t \Bigr|> q_{\text{max}}$. 

Now, assume that $\Bigl| w_t+\frac{\|u_{t-1}\|_2}{\|X_t\|_2}\cos\theta_t \Bigr|\leq q_{\text{max}}$. In this case, because the argument of $\mathcal{Q}$ lies in the active range of $\mathcal{A}$, we obtain 
\begin{equation}\label{lemma-diff-hard-eq4} 
\Bigl( w_t+\frac{\|u_{t-1}\|_2}{\|X_t\|_2}\cos\theta_t -q_t\Bigr)^2 
\leq \max\Bigl\{\lambda, \frac{\delta}{2}\Bigr\}^2. 
\end{equation}
Applying \eqref{lemma-diff-hard-eq3} and \eqref{lemma-diff-hard-eq4} to \eqref{lemma-diff-hard-eq1}, one can get 
\begin{equation}\label{lemma-diff-hard-eq5}
    \|P_{X_t}(u_t)\|_2^2\leq 
\begin{cases}
\frac{\max\{2\lambda, \delta\}^2}{4}\|X_t\|_2^2 &\text{if}\; \Bigl| w_t+\frac{\|u_{t-1}\|_2}{\|X_t\|_2}\cos\theta_t \Bigr|\leq q_{\text{max}}, \\
 \|u_{t-1}\|_2^2\cos^2\theta_t  &\text{otherwise}.
\end{cases}
\end{equation}
Again, by the same discussion after \eqref{lemma-diff-eq2} in \cref{lemma-diff}, we have 
$\|u_t\|_2^2 -\|u_{t-1}\|_2^2= \|P_{X_t}(u_t)\|_2^2-\|u_{t-1}\|_2^2\cos^2\theta_t$.
Replacing $\|P_{X_t}(u_t)\|_2^2$ with its upper bounds in \eqref{lemma-diff-hard-eq5}, we obtain \eqref{lemma-diff-hard-eq}.
\end{proof}

\begin{lemma}\label{lemma-bound-hard}
Let $\widetilde{\mathcal{A}}$ be one of the alphabets defined in \eqref{eq-alphabet-midtread-variant} with step size $\delta>0$, the largest element $q_{\max}$ and $\lambda\in(0,q_\mathrm{max})$. Suppose that $w\in\R^{N_0}$  satisfies $\|w\|_\infty\leq q_{\mathrm{max}}$, and consider the quantization scheme  given by \eqref{eq-update-rule-single-hard}. Additionally, denote the information of the first $t-1$ quantization steps by a $\sigma$-algebra $\mathcal{F}_{t-1}$, and let $\beta, \eta>0$, $s^2\in (0,1)$. Then the following results hold for $t=1,2,\ldots,N_0$.
\begin{enumerate} 
    \item 
    $\E e^{\eta\|u_{t}\|_2^2} \leq \max\Bigl\{
    \E(e^{\frac{\eta\max\{2\lambda,\delta\}^2}{4}\|X_t\|_2^2}e^{\eta\|u_{t-1}\|_2^2(1-\cos^2\theta_t)}),
    \E e^{\eta\|u_{t-1}\|_2^2} \Bigr\} $. 
    \item
    $\E(e^{\eta\beta\|u_{t-1}\|_2^2(1-\cos^2\theta_t)}\mid \mathcal{F}_{t-1}) \leq 
    -\E(\cos^2\theta_t\mid\mathcal{F}_{t-1})(e^{\eta\beta\|u_{t-1}\|_2^2}-1)+e^{\eta\beta\|u_{t-1}\|_2^2}$.
\end{enumerate}
Here, $\theta_t$ is the angle between $X_t$ and $u_{t-1}$.
\end{lemma}
\begin{proof}
Similar to \cref{lemma-bound}, the inequality (1) follows immediately from \eqref{lemma-diff-hard-eq}. The proof of part (2) is identical with the one in \cref{lemma-bound}.
\end{proof}
The proof of \cref{thm-cluster-sparse} is given as follows.
\begin{proof}
The only difference between \cref{lemma-bound} and its analogue \cref{lemma-bound-hard} is that $\delta^2$ in \cref{lemma-bound} is replaced by $\max\{2\lambda+\delta\}^2$. Note that \cref{lemma-bound} was used in the proof of both \cref{thm-main-result-single} and \cref{thm-cluster} in which $\delta^2$ serves as a coefficient. Hence, by substituting $\delta^2$ with $\max\{2\lambda+\delta\}^2$, it is not hard to verify that \cref{thm-cluster-sparse} holds.
\end{proof}

\end{document}